\documentclass[twoside]{article}

\usepackage{amsmath, amssymb, amsthm}
\usepackage{xcolor}
\usepackage{algorithm}
\usepackage{algpseudocode}
\newtheorem{theorem}{Theorem}[section]
\newtheorem{definition}{Definition}[section]

\usepackage{hyperref}
\usepackage{booktabs}
\usepackage{xspace}
\usepackage{subcaption}
\usepackage[colorinlistoftodos,textwidth=2.1cm]{todonotes}
\newcommand{\ab}[1]{\todo[color=green, size=\tiny, inline]{AB: #1}}
\newcommand{\g}[1]{\todo[color=cyan, size=\tiny, inline]{G: #1}}
\renewcommand{\ab}[1]{}
\renewcommand{\g}[1]{}

\usepackage[accepted]{aistats2026}

\usepackage[round]{natbib}

\DeclareMathOperator*{\range}{range}

\newcommand{\prevmethod}{Fed-BNSL\xspace}
\newcommand{\method}{Fed-Sparse-BNSL\xspace}
\newcommand{\dpmethod}{DP-Fed-Sparse-BNSL\xspace}
\newcommand{\dpprevmethod}{DP-Fed-BNSL\xspace}

\begin{document}

%
\runningtitle{Differentially Private and Federated Structure Learning in Bayesian Networks}

%

\twocolumn[

\aistatstitle{Differentially Private and Federated\\Structure Learning in Bayesian Networks}

\aistatsauthor{ Ghita Fassy El Fehri \And Aurélien Bellet \And  Philippe Bastien }

\aistatsaddress{PreMeDICal, Inria, Idesp,\\ Inserm, University of Montpellier \And PreMeDICal, Inria, Idesp,\\ Inserm, University of Montpellier \And L’Oréal Research and Innovation,\\  Aulnay-sous-Bois, France} ]

\begin{abstract}
\looseness=-1 Learning the structure of a Bayesian network from decentralized data poses two major challenges: (i) ensuring rigorous privacy guarantees for participants, and (ii) avoiding communication costs that scale poorly with dimensionality. In this work, we introduce \method, a novel federated method for learning linear Gaussian Bayesian network structures that addresses both challenges. By combining differential privacy with greedy updates that target only a few relevant edges per participant, \method efficiently uses the privacy budget while keeping communication costs low. Our careful algorithmic design preserves model identifiability and enables accurate structure estimation. Experiments on synthetic and real datasets demonstrate that \method achieves utility close to non-private baselines while offering substantially stronger privacy and communication efficiency.
\end{abstract}

\section{INTRODUCTION}
\looseness=-1 Bayesian networks (BNs) compactly represent joint distributions and causal dependencies among random variables. Their graph structure, in the form of a directed acyclic graph (DAG), encodes conditional independencies, yielding interpretable models widely used in biology, medicine, and social sciences. A key challenge is to infer the BN structure directly from data, a task known as Bayesian network structure learning (BNSL). Among existing approaches, continuous optimization methods have shown strong empirical performance on linear Gaussian BNs, where each variable is modeled as a linear function of its parents plus Gaussian noise \citep{notears,ng2022masked,zheng2020,golem,vowels}. However, these methods assume centralized access to data. In many real-world scenarios, data are distributed across multiple institutions or participants and cannot be aggregated due to privacy concerns. This motivates the study of federated BNSL, which aims to collaboratively reconstruct a global DAG from decentralized data \citep{fed-notears}.
Federated BNSL faces two key challenges: (i) ensuring strong privacy guarantees for participants' data, and (ii) preventing communication costs from growing excessively with dimensionality. These challenges are further compounded by the fact that the number of possible edges in a DAG grows quadratically with the number of variables.

In this work, we address both challenges with \method, a federated method for learning linear Gaussian Bayesian network structures, and its privacy-preserving variant \dpmethod.
Our approach combines differential privacy (DP) with a greedy, coordinate-wise update strategy that restricts communication to only a few relevant edges per participant at each iteration. This leads to two key advantages compared to prior work. First, communication costs are greatly reduced since sparse edge updates are transmitted instead of dense matrices. Second, privacy budget scales with the number of updated dependencies per participant, not with the full dimension. The core component of our method is a Proximal Greedy Coordinate Descent (PGCD) algorithm for the local updates, which produces high-quality sparse solutions without requiring data standardization---which is known to compromise the identifiability of the DAG \citep{ng,loh}. For the private variant, \dpmethod, we improve upon the Differentially Private Proximal Greedy Coordinate Descent (DP-PGCD) algorithm \citep{mangold} by minimizing the noise added under DP, leveraging mechanisms that can be analyzed under zero-concentrated differential privacy (zCDP) for tighter privacy accounting. This results in favorable privacy-utility trade-off, as demonstrated by both theoretical guarantees and empirical results. Besides the above benefits, we empirically show that \method supports participant-level personalization in heterogeneous settings, where the structure is shared but edge weights differ across participants. 

\textbf{Contributions.} Our main contributions can be summarized as follows: (i) a communication-efficient federated BNSL algorithm that exploits sparsity, reducing communication costs; (ii) a differentially private variant with formal $(\varepsilon, \delta)$-DP guarantees, that achieves strong utility even in high-dimensional settings; (iii) an empirical evaluation on synthetic homogeneous and heterogeneous data as well as a real dataset, demonstrating the effectiveness of our method in terms of convergence, communication costs, privacy-utility trade-offs, robustness to dimensionality, and effective participant-level personalization.

\textbf{Organization of the paper.} The paper is organized as follows. Section \ref{sec:pb-setting} formalizes the problem setting; Section \ref{sec:related-work} reviews related work; Section \ref{sec:fed-alg} presents the federated algorithm \method; Section \ref{sec:DP-alg} introduces the private variant \dpmethod and its guarantees; and finally, Section \ref{sec:expe} reports experimental results, and Section \ref{sec:conclu} concludes. 

\section{PROBLEM SETTING}
\label{sec:pb-setting}

\textbf{Structure learning in linear Gaussian BNs.}
\looseness=-1 We consider the task of learning the structure of a Bayesian Network (BN) from observational data. A BN \citep{BN} is a directed acyclic graph (DAG) where each node $X_i$ represents a random variable, and edges encode conditional dependencies. We represent a (weighted) DAG using a weighted adjacency matrix $W \in \mathbb{R}^{d \times d}$, where each entry $W_{i,j}$ encodes the strength of a directed edge from variable $X_i$ to $X_j$. A nonzero $W_{i,j}$ indicates a direct causal effect of $X_i$ on $X_j$, whereas $W_{i,j} = 0$ indicates no direct effect.

We focus on linear Gaussian BNs, where each variable is a linear function of its parents in the DAG plus additive Gaussian noise:
\[
X_i = \sum_{j \in \mathrm{Pa}(X_i)} W_{ji} X_j + Z_i,
\qquad Z_i \sim \mathcal{N}(0, \sigma_i^2),
\]
where $\mathrm{Pa}(X_i) = \{X_j : W_{j,i} \neq 0\}$ denotes the set of parents of $X_i$ in the DAG.

The structure learning problem is to recover $W$ from a dataset $X \in \mathbb{R}^{n \times d}$, which is known to be identifiable when noise variances are equal or known \citep{identif,loh}.
The combinatorial nature of the problem is addressed by reformulating it as a continuous optimization task using the NOTEARS framework \citep{notears}:
\begin{equation}
\label{eq:notears}
\min_{W \in \mathbb{R}^{d \times d}} L(W; X) + \lambda \|W\|_1
\qquad \text{s.t.} \quad h(W) = 0, 
\end{equation}
where $L(W; X) = \frac{1}{2n}\| X - XW\|_F^2$ is the least squared loss, $\lambda \|W\|_1$ promotes sparsity, and 
\[
h(W) = \operatorname{tr}(e^{W \circ W}) - d = 0
\]
ensures acyclicity, where $\circ$ is the Hadamard product. 

\textbf{Federated learning.}
In this work, we consider a setting with $P$ participants, each holding a private local dataset $X^{(p)} = \{X_1^{(p)}, \dots, X_{n_p}^{(p)}\} \subset \mathbb{R}^d$, with $n_p$ the number of samples held by participant $p$. The goal is to learn a shared DAG structure with the coordination of a central server, while keeping all data decentralized.

Problem \eqref{eq:notears} can then be equivalently reformulated as a consensus problem \citep{fed-notears}:
\begin{align}
    \textstyle\min_{B, W} \; \quad & \textstyle\mathcal{L}(B; W) = \sum_{p=1}^P L(B_p; X^{(p)}) + \lambda \|W\|_1\nonumber \\
    \text{s.t. } \quad &  B_p = W,  \qquad \forall p \in  \{1, \dots, P\},\nonumber \\
    & h(W) = 0,
\label{eq:fed-admm-base}
\end{align}
where $B=[B_1,\dots,B_P]$ and $B_p\in\mathbb{R}^{d\times d}$ is the local matrix for participant $p$ and $W$ the global consensus adjacency matrix. 

\textbf{Differential privacy.}
To protect the participants' data in the federated learning process, we aim to enforce formal differential privacy (DP) guarantees \citep{dwork1}. Intuitively, a mechanism satisfies DP if replacing a single data point has only a limited impact on the algorithm's output distribution. 
\begin{definition}[Differential privacy]
    Let $\varepsilon, \delta > 0$. A randomized algorithm $\mathcal{A}$ satisfies $(\varepsilon, \delta)$-DP if, for any two datasets $D_1$ and $D_2$ of fixed size that differ in exactly one record, and for any possible $O\subseteq\range(\mathcal{A})$,
    \[
    P(\mathcal{A}(D_1) \in O) \leq e^{\varepsilon} P(\mathcal{A}(D_2) \in O) + \delta.
    \]
\end{definition}

In our federated setting, we aim to limit the influence of any single participant's data on the information shared during training, thereby providing formal privacy guarantees against both external observers and an honest-but-curious server.

\section{RELATED WORK}
\label{sec:related-work}

\textbf{Bayesian networks structure learning (BNSL)} aims at recovering the graph structure of a BN from observational data \citep{BN}. Traditional structure learning approaches are based on discrete optimization and can be categorized into two families: constraint-based methods, which rely on conditional independence tests, and score-based methods, which optimize a goodness-of-fit score. For a detailed survey, 
we refer the reader to \cite{kitson}.

Recently, continuous optimization methods have gained popularity, notably with the introduction of NOTEARS \citep{notears}, which proposed the first differentiable acyclicity constraint in the context of linear Gaussian BNs. This allows BNSL to be formulated as a continuous optimization problem, avoiding the combinatorial search of discrete methods. Several extensions of NOTEARS have since been proposed \citep{ng2022masked,zheng2020,golem}, see also 
\citet[][Section 5]{vowels}. 
However, these methods require centralized access to data.

\textbf{Federated BNSL} extensions have been developed for both continuous and discrete optimization approaches. Discrete optimization methods \citep{vertibayes,FedGES} are primarily designed for discrete data and are not directly applicable to continuous data, which is the focus of our work.
For continuous optimization, \prevmethod \citep{fed-notears} adapts the NOTEARS formulation to the federated setting via the alternating direction method of multipliers (ADMM) to solve \eqref{eq:fed-admm-base}, enabling participants to collaboratively learn a consensus DAG without centralizing their data. FedDAG \citep{fedDAG}, built on the GOLEM framework \citep{golem}, relies on FedAvg \citep{fedavg} to learn and aggregate local structures across participants. Both methods require participants to transmit $d \times d$ matrices at each round, which becomes impractical in high-dimensions, and neither provides formal privacy guarantees.


\textbf{Differentially private BNSL} methods have been proposed in the context of synthetic data generation. PrivBayes \citep{privbayes} and subsequent work \citep{ergute} propose approaches to learn BNs under DP, enabling the release of synthetic data with formal privacy guarantees. Related efforts have appeared in the causal discovery community \citep{NEURIPS2020_3b13b1eb}. However, these methods are designed for discrete data, so they cannot be applied to our continuous setting, and assume centralized access to all data. 

\textbf{Federated and private BNSL} has been explored in recent work. In \citep{regrets}, participants compute scores on local DAGs and send noisy scores to a central server, which then identifies a consensus DAG minimizing a notion of regret. However, the method lacks explicit sensitivity bounds for DP guarantees and does not scale, requiring as many communication rounds as graph edges. \citet{wang} extend constraint-based methods to the federated setting by performing conditional independence tests locally and aggregating results via secure computation, but this approach is again limited to discrete data.
\section{FEDERATED ALGORITHM}
\label{sec:fed-alg}


In existing federated BNSL approaches, 
each participant must transmit a dense $d \times d$ matrix to the central server \citep{fed-notears,fedDAG}, resulting in a communication cost that scales quadratically with the number of variables, despite the underlying DAG typically being sparse. In this section, we propose a communication-efficient alternative that exploits this sparsity, which is also crucial for the design of an effective differentially private version (Section~\ref{sec:DP-alg}).

\subsection{Algorithm Overview}

We first slightly modify the formulation \eqref{eq:fed-admm-base} of \citet{fed-notears}, redefining the objective function $\mathcal{L}(B, W)$ by transferring the $\ell_1$ sparsity penalty from the global consensus matrix $W$ to the local matrices:
\begin{align}
    \textstyle\min_{B, W} \; \quad & \textstyle\mathcal{L}(B; W) = \sum_{p=1}^P L(B_p; X^{(p)}) + \lambda \|B_p\|_1\nonumber \\
    \text{s.t. } \quad &  B_p = W,  \qquad \forall p \in  \{1, \dots, P\},\nonumber \\
    & h(W) = 0.
\label{eq:FSB-form}
\end{align}
Note that this reformulation yields an equivalent problem because of the consensus constraints in \eqref{eq:fed-admm-base}. Following the approach of \prevmethod, we solve this constrained optimization problem in a federated way using the ADMM-based augmented Lagrangian method, leading to the following iterative update rules:
\begin{align}
     B_p^{t+1} & = \operatorname*{arg\,min}_{B_p} (L(B_p ;X^{(p)}) + \lambda \|B_p\|_1 \label{eq:Bp}\\
    & + \operatorname{tr}\left(\beta_p^t (B_p - W^t)^\top\right) 
    + \frac{\rho_2}{2} \|B_p - W^t\|_F^2),  \notag  \\
    W^{t+1} & = \textstyle\operatorname*{arg\,min}_W (\sum_{p=1}^P \operatorname{tr}\left(\beta_p^t (B_p^{t+1} - W)^\top\right) \label{eq:W} \\
    &\textstyle + \frac{\rho_2^t}{2} \sum_{p=1}^P \|B_p^{t+1} - W\|_F^2
    + \alpha^t h(W) + \frac{\rho_1}{2} h(W)^2), \notag \\
    \alpha^{t+1} & = \alpha^t + \rho_1 h(W^{t+1}), \label{eq:alpha}\\
     \beta_p^{t+1} & = \beta_p^{t} + \rho_2(B_p^{t+1} - W^{t+1}), \label{eq:beta}
\end{align}

where $\rho_1, \rho_2 > 0$ are the penalty coefficients and $\alpha \in \mathbb{R}$ and $\beta_1, \dots, \beta_P \in \mathbb{R}^{d \times d}$ are the dual variables (Lagrange multipliers). Algorithm \ref{algo:fed-BNSL} summarizes the proposed federated BNSL algorithm.


Imposing the $\ell_1$ penalty directly on the local matrices leverages the inherent sparsity of Bayesian networks to reduce communication. Since solutions to the local subproblem \eqref{eq:Bp} are expected to be sparse, participants only need to exchange the nonzero entries of their estimates, reducing the communication cost from $O(d^2)$ to the number of identified dependencies. This sparsity propagates to the shared structure $W$: for all entries that are zero across the local matrices, any local minimizer of \eqref{eq:W} will likewise assign a zero to that entry.

\begin{algorithm}[t]
    \caption{\method}
    \label{algo:fed-BNSL}
    \begin{algorithmic}[1]
        \State Initialize $B_p^0=0, W^0=0, \alpha^0=0, \beta_p^0=0$
        \For{$t = 1$ \textbf{to} $T$}
            \For {each participant $p$ in parallel}
                \State Update $B_p^{t+1}$ by solving Eq. \ref{eq:Bp} with PGCD
                \State Send sparse update $B_p^{t+1}$ to server
            \EndFor
            \State Update $W^{t+1}$ by solving Eq. \ref{eq:W} with L-BFGS
            \State Update dual variables $\alpha^{t+1}, \beta_p^{t+1}$ (Eq. \ref{eq:alpha}-\ref{eq:beta})
            \State Server sends back sparse $W^{t+1}$ to participants
        \EndFor
    \end{algorithmic}
\end{algorithm}

It remains to select the appropriate local solver for \eqref{eq:Bp}, which is a subtle but crucial design choice.

\subsection{Choice of Local Solver}
\label{sec:local_solver}
The local subproblem \eqref{eq:Bp} is a LASSO-type optimization problem, for which many solvers exist \citep{LASSO,FW}. However, these solvers typically assume that data has been standardized (features centered and rescaled). Standardization ensures that all variables are on a comparable scale, allowing the $\ell_1$ penalty to be applied uniformly across coefficients~\citep{LASSO}.
However, in the context of linear Gaussian structural equation model with equal error variances \citep{identif}, standardization is known to invalidate the theoretical conditions under which the true DAG can be identified \citep{ng,loh}
We therefore require a solver that can meaningfully handle LASSO problems without rescaling the data. 

For this reason, we rely on the \emph{Proximal Greedy Coordinate Descent} (PGCD) algorithm \citep{tseng2009Coordinate,nutini2015Coordinate,karimireddy2019Efficient}. 
PGCD is an iterative method that updates one coordinate at a time, selecting the coordinate that yields the greatest potential improvement in the objective. Consider participant $p$ with initialization $B^0\in\mathbb{R}^{d\times d}$. At each iteration $k$, PGCD evaluates each coordinate $(i,j)$ using the score
\begin{align}
\begin{aligned}
     S_{i,j} &=  \sqrt{M_{i,j}} \Big| \operatorname{prox}_{\frac{\lambda}{M_{i,j}} \lvert \cdot \rvert} \Big( B_{i,j}^{k} \\
     &- \frac{1}{M_{i,j}} \big( \nabla_{i,j} {\mathcal{L}}(B^{k};X^{(p)})\big)\Big) - B_{i,j}^{k} \Big|,
\label{eq:score-non-private}
\end{aligned}
\end{align}
where $\mathcal{L}(B^{k};X^{(p)})=L(B^{k};X^{(p)}) + \operatorname{tr}(\beta (B^{k} - W)^\top) + \frac{\rho}{2} \|B^{k} - W\|_F^2$ is the smooth part of the local objective, $\nabla_{i,j} {\mathcal{L}}(B^{k};X^{(p)})$ is the partial derivative of $\mathcal{L}$
with respect to $B^{k}_{i,j}$, $\operatorname{prox}_{\lambda \lvert \cdot \rvert}(x) = \operatorname{sign}(x) \cdot \max(|x| - \lambda, 0)$ is the soft-thresholding operator, and $M_{i,j}$ is the coordinate-wise smoothness constant of $\mathcal{L}$.
The coordinate with the highest score, $(m,n) = \arg\max_{i,j} S_{i,j}$, is then greedily updated by a proximal gradient step:  
\begin{equation}
\label{eq:update}
    B_{m,n}^{k+1} = \operatorname{prox}_{\frac{\lambda \gamma}{M_{m,n}} \lvert \cdot \rvert} \big( B^{k}_{m,n} 
    - \frac{\gamma}{M_{m,n}} \nabla_{m,n} {\mathcal{L}}(B^{k};X^{(p)})\big)
\end{equation}

PGCD is particularly well-suited to our setting for two reasons. First, the coordinate-wise smoothness constants naturally normalize updates, eliminating the need for standardization and thereby preserving the identifiability conditions of the underlying DAG. We empirically validate the importance of this property in Appendix~\ref{app:std}, where we show that classical solvers fail on raw data and that standardization compromises structure recovery.
Second, the greedy coordinate selection combined with proximal updates enforces sparsity, retaining only the most relevant dependencies, thereby significantly reducing communication costs.

Furthermore, because the $\ell_1$-regularized local subproblem \eqref{eq:Bp} satisfies both coordinate-wise smoothness and strong convexity, the PGCD algorithm benefits from a linear convergence guarantee to the exact local optimum \citep[Theorem 1]{karimireddy2019Efficient}.


In the next section, we further leverage the properties of PGCD to ensure the privacy budget is effectively spent on the most informative updates---a key factor for achieving good privacy-utility trade-offs in high-dimensional settings.

\section{PRIVATE ALGORITHM}
\label{sec:DP-alg}

\subsection{Motivation}

In the original \prevmethod algorithm \citep{fed-notears}, the local subproblem in $B_p$—corresponding to \eqref{eq:Bp} without the $\ell_1$ penalty—is solved via the closed-form expression:
\begin{equation}
\label{eq:closed-form}
B_p^{t+1} = (\Sigma_p + \rho_2 I)^{-1}(\rho_2 W^t - \beta_p^t + \Sigma_p),
\end{equation}
where $\Sigma_p = \frac{1}{n_p}X^{(p)\top} X^{(p)}\in\mathbb{R}^{d\times d}$ denotes the empirical covariance matrix of participant $p$. Since the server receives $B_p$ and has access to $W^t$, $\rho_2$, and $\beta_p^t$, it can directly reconstruct $\Sigma_p$, as detailed in Appendix \ref{app:cov}. This poses a significant privacy risk: covariance matrices can leak detailed information about individual data points, enabling attribute inference and data reconstruction attacks \citep{doi:10.1126/sciadv.adj9260,10.1093/bioinformatics/btad531}.
Although methods exist for constructing differentially private covariance matrices \citep{DBLP:conf/uai/Wang18,DBLP:conf/nips/AminDKMV19}, they become impractical for high-dimensional datasets unless additional assumptions are made, since the privacy-induced error grows quadratically with the number of variables $d$ \citep{DBLP:conf/nips/AminDKMV19}.

Our approach, presented in Section~\ref{sec:fed-alg}, naturally produces sparse updates for $B_p$, which reduces the amount of information exposed to the server. Nevertheless, it is well known that sharing model updates in federated learning can still be vulnerable to a range of privacy attacks \citep{Shokri2019,inverting_gradients}. In this section, we introduce a differentially private variant of our federated BNSL algorithm that leverages the greedy updates of PGCD, avoiding the quadratic cost in the dimensionality $p$.


\subsection{Differentially Private Local Solver}

\begin{algorithm}[t]
    \caption{DP-PGCD}
    \label{algo:DP-PGCD}
    \begin{algorithmic}[1]
        \State Initialize $B^0=0$ (or use warm-start)
        \For{$k = 1$ \textbf{to} $K$}
        \State Compute noisy scores $\tilde S_{i,j}$ $\forall (i,j)$ (Eq. \ref{eq:noisy-score})
        \State Select $(m,n)$ with highest noisy score
        \State Update $B_{m,n}^{k+1}$ (Eq. \ref{eq:noisy-update})
        \State $B_{i,j}^{k+1} = B_{i,j}^k$ for $(i,j) \neq (m,n)$  
        \EndFor
        \State \textbf{Output:} Sparse matrix $B^K$
    \end{algorithmic}
\end{algorithm}

In \method, the only part that directly accesses the data is the participants’ local update \eqref{eq:Bp}. Therefore, it is sufficient to privatize this step; the post-processing and composition properties of differential privacy then provide guarantees for the overall algorithm.
We thus propose to rely on a differentially private version of our local solver, PGCD.

\citet{mangold} introduced a private version of PGCD which relies on the addition of Laplace noise to privatize both the greedy coordinate selection and the gradient update, providing pure $\epsilon$-DP guarantees, and resorted to the advanced composition theorem \citep{dwork2} to track the privacy loss across iterations. However, this composition is overly pessimistic, requiring unnecessarily large amounts of noise at each step and resulting in a poor privacy–utility trade-off, as the excessive noise degrades the quality of the learned matrix.

In contrast, our improved DP-PGCD uses the exponential mechanism via the ``Gumbel max trick'' \citep{DBLP:journals/corr/abs-2105-07260} for private coordinate selection and Gaussian noise for gradient updates. This design enables us to leverage zero-Concentrated Differential Privacy (zCDP) \citep{10.1007/978-3-662-53641-4_24}, under which the exponential mechanism admits a tighter privacy analysis \citep{pmlr-v119-dong20a} and composition behaves more favorably. We thus achieve tighter privacy accounting and require less noise for the same privacy budget. A detailed theoretical and numerical comparison with the approach of \citet{mangold} is provided in Appendices~\ref{app:DPPGCD-comparison} and \ref{app:DPPGCD-expe}, demonstrating a provable improvement in privacy loss of at least a factor of $\sqrt{2}$ for the same noise variance.

Our DP-PGCD algorithm (Algorithm~\ref{algo:DP-PGCD}) operates similarly to its non-private counterpart, with two key modifications. First, it uses noisy scores computed by adding independent Gumbel noise to the non-private scores defined in \eqref{eq:score-non-private}:
\begin{equation}
\label{eq:noisy-score}
\tilde S_{i,j} = S_{i,j} + \mathrm{Gumbel}( 0, \beta).
\end{equation}
Second, the gradient used in the
%
coordinate update \eqref{eq:update} is perturbed with Gaussian noise:
\begin{align}
\begin{aligned}
     B_{m,n}^{k+1} &= \operatorname{prox}_{\frac{\lambda \gamma}{M_{m,n}} \lvert \cdot \rvert} \big( B^{k}_{m,n} \\
     & - \frac{\gamma}{M_{m,n}} (\nabla_{m,n} {\mathcal{L}}(B^{k};X^{(p)}) + \mathcal{N}( 0,\sigma^2))\big).
\end{aligned}
\label{eq:noisy-update}
\end{align}
Note that this private version reduces to the non-private algorithm when $\beta = \sigma = 0$.



Each iteration of DP-PGCD computes the full gradient but updates only a single coordinate. While using the full gradient would require privatizing all $d^2$ coordinates, even though only a few are typically relevant in sparse DAGs, the exponential mechanism allows PGCD to select the most promising coordinate while incurring a privacy cost of only $O(\log d)$ \citep{dwork2}. In the following, we refer to \dpmethod as the \method algorithm using DP-PGCD as the local solver.



\subsection{Privacy Guarantees}

We provide DP guarantees for \dpmethod.


\begin{theorem}[Privacy of \dpmethod]
\label{thm:fed-dp}
    Let $\varepsilon, \delta > 0$ and $\Delta=2L_{i,j}/n_p$ where $L_{i,j}$ is the coordinate-wise Lipschitz constant of $\mathcal{L}$.
    Suppose \dpmethod (Algorithm \ref{algo:fed-BNSL}) runs for $T$ global rounds, with local updates performed with $K$ iterations of DP-PGCD (Algorithm~\ref{algo:DP-PGCD}). If the Gumbel and Gaussian noise parameters are chosen as
\[
    \beta = \frac{\sigma}{\sqrt{M_{i,j}}} =  \frac{\Delta \sqrt{KT}(\sqrt{\log(\frac{1}{\delta}) + \varepsilon} +\sqrt{\log(\frac{1}{\delta})})}{\varepsilon},
\]
then \dpmethod is $(\varepsilon, \delta)$-differentially private with respect to each participant's dataset.
\end{theorem}
\begin{proof}[Sketch of proof]
We first show that gradients have $\ell_2$-sensitivity bounded by $\Delta$, and that scores have sensitivity bounded by $\Delta/\sqrt{M_{i,j}}$, leveraging the non-expansiveness of the $\operatorname{prox}$. Each DP-PGCD iteration applies two private mechanisms: (i) coordinate selection via the exponential mechanism (with Gumbel noise), and (ii) coordinate update via Gaussian noise. With $\beta = \sigma/\sqrt{M_{i,j}}$, each one satisfies $\rho$-zCDP, where $\rho = \tfrac{\Delta^2}{2\beta^2}$. By additive composition, running $T$ rounds of \dpmethod, each with $K$ updates of DP-PGCD, yields total privacy loss $\rho_{\mathrm{tot}} = 2K\rho = \tfrac{K\Delta^2}{\beta^2}$. Converting to $(\varepsilon,\delta)$-DP gives $\varepsilon = \rho_{\mathrm{tot}} + 2\sqrt{\rho_{\mathrm{tot}} \log(1/\delta)}$, which determines the stated noise scales. Full derivations are in Appendix \ref{app:privacy-proof}.
\end{proof}


\textbf{Gradient clipping.} Since coordinate-wise Lipschitz constants are difficult to bound tightly, in practice we instead enforce bounded sensitivity through gradient clipping. Following \citet{Mangold2022a}, each gradient coordinate $\nabla_{i,j} {\mathcal{L}}(B^{k};X^{(p)})$ is clipped at $C_{i,j}=\sqrt{M_{i,j}/\sum_{i,j}M_{ij}}C$ for some global threshold $C>0$. This scheme adjusts for the relative scale of each coordinate while reducing tuning to a single parameter.

\textbf{Knowledge of $M_{i,j}$.} We assume that the coordinate-wise smoothness constants $M_{i,j}$ are available (or upper bounded). When this is not the case, they can be privately estimated using the Gaussian mechanism. Further details are provided in Appendices~\ref{app:smoothness_theo} and \ref{app:smoothness_expe}.

\subsection{Extension to More General Bayesian Networks}
While this work focuses on linear Gaussian BNs—a setting where DAG identifiability is well understood and strong baselines allow us to isolate the contribution of our FL/DP design—our framework is not fundamentally limited to this class.

Our approach can be naturally extended to broader classes of BNs, provided the following conditions hold: (i) the underlying DAG is identifiable under suitable assumptions, (ii) the global loss $L(W;X)$ decomposes as a sum of local losses $L(B_p; X^{(p)})$ across participants, (iii) the smooth part of each local score $L(B_p; X^{(p)})$ is differentiable and coordinate-wise smooth (so that the PGCD solver can be applied), and (iv) the corresponding gradients can be bounded to derive sensitivity bounds for the DP mechanisms.

When these conditions are satisfied, extending our framework only requires replacing the squared loss with the appropriate score function and re-deriving the associated smoothness constants and sensitivity bounds.

\section{EXPERIMENTS}
\label{sec:expe}

We empirically evaluate \method and its differentially private variant \dpmethod on synthetic data and on real data. We assess the effectiveness of our method in terms of convergence, structural accuracy of the estimated DAG, communication efficiency and privacy-utility trade-offs. The code is available at \url{https://gitlab.com/ghitafassy/fed-sparse-bnsl/}.

\subsection{Experimental Setup}

\textbf{Datasets} We consider the following datasets.

\textit{Homogeneous synthetic data.} Following \citet{notears} and \citet{fed-notears}, we generate Erdös-Rényi random DAGs with $d=20$ nodes and an expected number of $d$ edges. Observations are drawn from a linear Gaussian BN with equal noise variance across variables. Data are partitioned across $P=8$ participants, each with $n_p=5000$ samples. We also consider higher-dimensional versions with $d$ up to $200$. Details are provided in Appendix \ref{app:data_homo}.

\textit{Heterogenous synthetic data.} We use the same DAG generation, but allow participant-specific regression coefficients, while keeping the underlying DAG structure shared. We consider $P=5$ participants, each with $n_p=5000$ samples.
This setting allows to evaluate both the structural accuracy of the estimated shared DAG, and the ability to estimate participant-specific coefficients.
Details are in Appendix \ref{app:data_hetero}.

\begin{figure*}[t]
    \centering  
    \begin{subfigure}[t]{0.8\textwidth}
    \centering
    \includegraphics[width=1\textwidth]{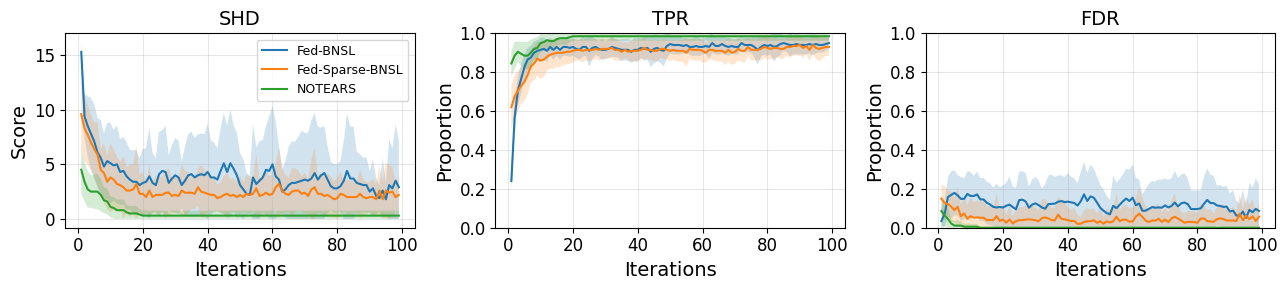} 
    \label{fig:non-DP}
    \end{subfigure}
    \begin{subfigure}[t]{0.8\textwidth}
    \centering
    \vspace*{-.3cm}
    \includegraphics[width=1\textwidth]{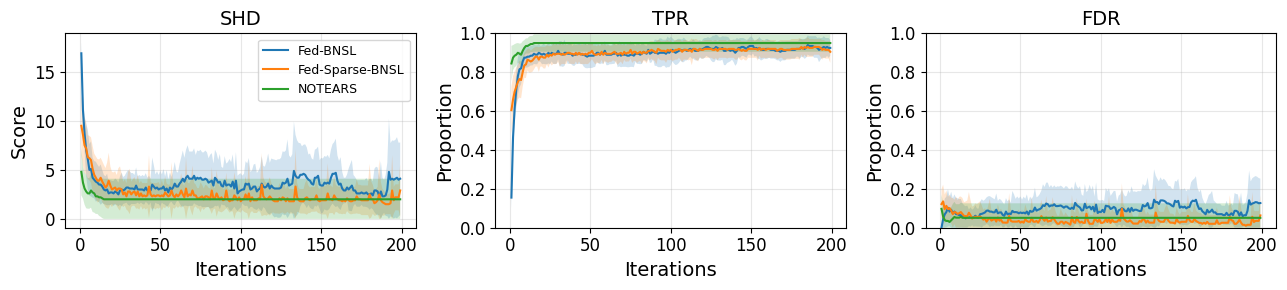} 
    \label{fig:cv_hetero}
    \end{subfigure}
    \caption{Convergence of \method, \prevmethod and centralized NOTEARS. Top: homogeneous synthetic data; Bottom: heterogeneous synthetic data. We report SHD, TPR and FDR across iterations.}
    \label{fig:conv}
\end{figure*}

\textit{Real data.} As in previous work \citep{notears,fed-notears}, we use the Sachs protein signaling dataset \citep{Sachs}, which has $d=11$ variables, $n=7466$ samples, and a ground-truth DAG with $18$ edges. To simulate a federated setting, we split the data across $P=3$ participants.

\textbf{Metrics.} We consider the following metrics.

\textit{Structural accuracy}. We assess the quality of the estimated DAGs using standard structure learning metrics. The Structural Hamming Distance (SHD) counts the number of edge insertions, deletions or reversals needed to convert the estimated DAG into the true one.
The True Positive Rate (TPR) represents the proportion of true edges that are correctly recovered, while the False discovery rate (FDR) represents the proportion of incorrect edges among predicted edges.

\textit{Communication cost.} We measure the total cost as the number of bytes transmitted between the server and the participants throughout the entire training procedure, reported in megabytes (MB).

\textit{Personalization.} In the heterogeneous setting, after learning the shared structure, we compute per-participant normalized mean squared error (MSE) between the true regression coefficients and those estimated after local refitting. 

\textbf{Evaluation protocol.}
We report the mean and standard deviation (shaded bands in figures) across the $10$ runs. For each configuration of dataset and privacy budget, hyperparameters are tuned separately for each method. The procedure, search ranges and final values are detailed in Appendix \ref{app:hyperparams}.\ab{organiser les appendix pour que les refs soient bonnes}

\subsection{Non-Private Setting}

We evaluate the performance of \method on synthetic data and compare to \prevmethod \citep{fed-notears}. We also include centralized NOTEARS as a reference point, which corresponds to running NOTEARS on the union of all participants' data and thus represents the best achievable performance.

\textbf{Convergence on homogeneous data.}
Figure~\ref{fig:conv} (top) shows that \method consistently achieves lower SHD than \prevmethod across iterations and stabilizes earlier, indicating faster and more stable convergence. Variability is also narrower for \method. Both methods quickly reach high TPR and remain close thereafter, but \method maintains a lower FDR throughout training. This gap persists over iterations and is accompagnied by tighter variability, suggesting that \method avoids over-selecting spurious edges while converging to more accurate structures.

\textbf{Convergence on heterogeneous data.} 
As shown in Figure \ref{fig:conv} (bottom), \method maintains its advantages on heterogeneous data: it achieves lower SHD with earlier stabilization, matches \prevmethod in TPR, and consistently attains lower FDR, demonstrating accurate and robust recovery of the consensus structure despite cross-participant variability.
\begin{table*}[t]
\centering
\small
\begin{tabular}{lccccc}
\toprule
Method & Dimension $d$ & Communication cost & SHD & TPR & FDR \\
\midrule
\method      & $20$ & $1.99 \pm 0.28$ & $2.2 \pm 1.81$ & $0.93 \pm 0.048$ & $0.057 \pm 0.078$\\
\prevmethod  & $20$ & $5.12 \pm 0$ & $2.9 \pm 4.483$ & $0.95 \pm 0.033$ & $0.086 \pm 0.145$ \\
\midrule
\method  & $200$ & $13.7 \pm 0.8$ & $27.4 \pm 14.516$ & $0.922 \pm 0.03$ & $0.084 \pm 0.05$ \\
\prevmethod & $200$ & $512 \pm 0$ & $50.7 \pm 12.859$ & $0.934 \pm 0.007$ & $0.19 \pm 0.045$ \\
\bottomrule
\end{tabular}
\caption{Communication costs (MB) and structural accuracy for $d=20$ and $d=200$.}
\label{tab:comm}
\end{table*}

\textbf{Communication efficiency.} 
Communication costs are reported in Table \ref{tab:comm}.
For $d=20$, \method reduces communication by approximately $64\%$ compared to \prevmethod, while maintaining strong structural accuracy. As expected, for $d=200$, the difference is even more pronounced: \method achieves roughly $97.2\%$ reduction in communication costs and clearly outperforms \prevmethod, with lower SHD and FDR while keeping TPR high. These results highlight the effectiveness of \method’s design, demonstrating substantial communication savings without sacrificing performance.
\begin{figure*}[t]
    \centering  
\includegraphics[width=.8\textwidth]{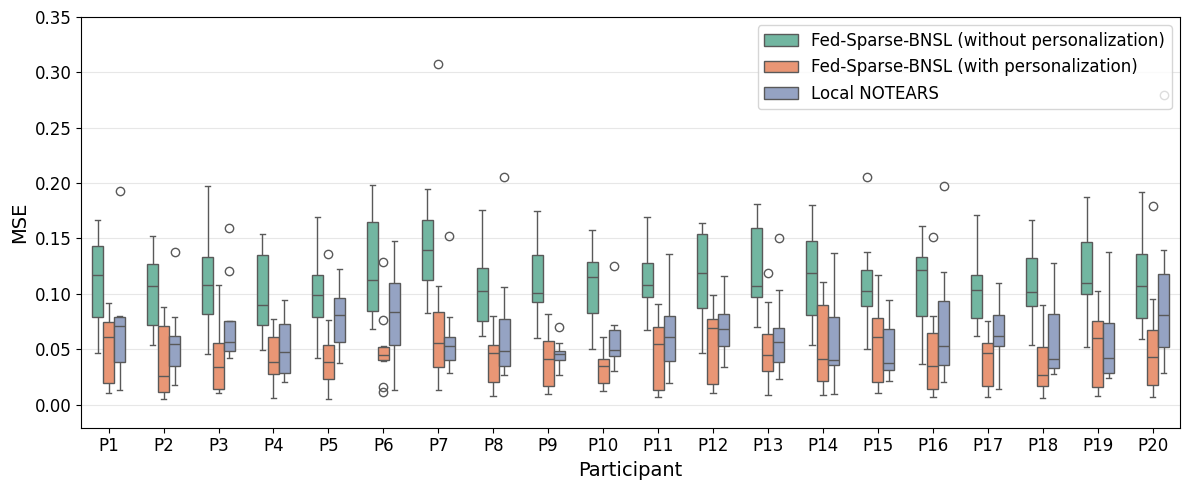} 
\caption{Participant-level personalization: per-participant normalized MSE between true and estimated parameters, for \method with and without personalization, compared to local NOTEARS.\label{fig:MSE_fed_solo}}
\end{figure*}

\textbf{Participant-level personalization.}
For this experiment, we consider a low-sample
setting with $20$ participants, each holding only $50$ samples.
We compare three approaches: (i) \method without personalization, using the consensus weighted adjacency matrix learned collectively, (ii) \method with personalization, where each participant locally refits the edge weights via linear regression on its own data after fixing the learned DAG structure and (iii) local NOTEARS, where each participant independently runs centralized NOTEARS on its own data. Figure \ref{fig:MSE_fed_solo} shows the per-participant normalized MSE between estimated and ground-truth participant-specific weights for each approach, showing that personalization consistently reduces MSE across participants, with lower medians in all boxplots. This demonstrates that, despite heterogeneity in edge weights, learning the network structure collectively provides a strong foundation, and subsequent local refitting produces participant-specific estimates that closely align with the ground truth.

Furthermore, while local NOTEARS already achieves reasonable MSE across participants, indicating that client-level estimates are not degenerate, \method with personalization consistently achieves the lowest median errors. This performance gap is deeply rooted in the quality of the underlying learned DAG: as detailed in Appendix~\ref{app:local_NOTEARS} (Figure~\ref{fig:local_vs_fed}), \method attains significantly better structural metrics than local NOTEARS. This confirms that aggregating information across participants improves structural recovery beyond what is achievable locally.

\subsection{Private Setting}

\begin{figure*}[t]
    \centering  
\includegraphics[width=0.8\textwidth]{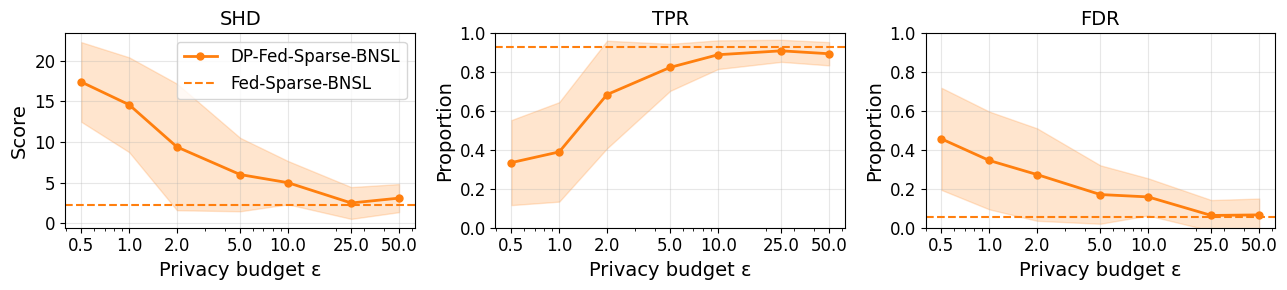} 
\caption{Privacy-utility trade-off: SHD, TPR and FDR of DP-\method under varying privacy budgets $\varepsilon$, compared to non-private \method.}
    \label{fig:dp_eps}
\end{figure*}

\begin{figure*}[t]
    \centering  
    \includegraphics[width=0.8\textwidth]{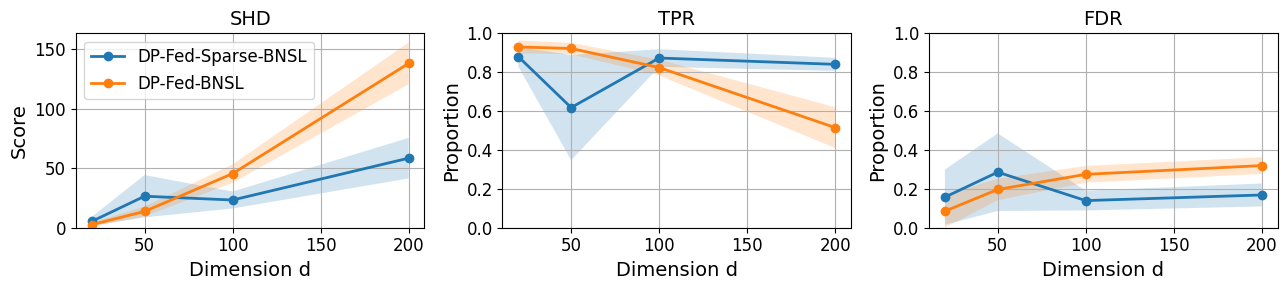} 
    \caption{Dimensional robustness: performance of \dpmethod vs DP-\prevmethod as dimension $d$ increases, under fixed privacy budget $\varepsilon=10$.}
    \label{fig:DP_dim}
\end{figure*}

\textbf{Privacy-utility trade-off.}
We compare our non-private algorithm (\method) to its differentially private variant (\dpmethod) on homogeneous synthetic data, varying the privacy budget $\varepsilon \in [0.5, 1, 2, 5, 10, 25, 50]$ while keeping $\delta=\frac{1}{n_p^2}$ fixed.
The results are shown in Figure \ref{fig:dp_eps}, where \method is shown as a dashed reference to visualize the utility gap under privacy.
In the high-privacy regime ($\varepsilon \leq 2$), \dpmethod exhibits higher SHD and lower TPR, while FDR is highest. As $\varepsilon$ increases, SHD decreases and TPR rises monotonically, approaching the non-private baseline. For moderate to large budgets ($\varepsilon \geq 5$), SHD and TPR are close to \method, variance shrinks, and FDR drops toward the non-private reference. Overall, increasing $\varepsilon$ narrows the privacy-utility gap, with $\varepsilon = 5$ providing a strong compromise and $\varepsilon = 10$ nearly matching the non-private performance.

We evaluate the effect of dimensionality under differential privacy by comparing \dpmethod to a baseline, DP-\prevmethod, in which each participant privatizes its covariance matrix using the Gaussian mechanism \citep{DBLP:conf/uai/Wang18} before running \prevmethod. Using homogeneous synthetic data, we vary the dimension $d \in [20, 50, 100, 200]$ while keeping the number of participants $P = 8$ and the per-participant sample size $n_p = 5000$ fixed. The privacy budget is set to $\varepsilon = 10$ for all dimensions.

As shown in Figure~\ref{fig:DP_dim}, SHD increases with $d$ for both methods, reflecting the growth in the number of edges. At small dimensions ($d = 20, 50$), DP-\prevmethod performs similarly to \dpmethod. However, its structural accuracy deteriorates rapidly as $d$ increases. At $d = 200$, DP-\prevmethod’s SHD is roughly twice that of \dpmethod (about 140 vs. 70), and its TPR drops to around 0.5 compared to over 0.8 for \dpmethod. This is expected, since privatizing full covariance matrices scales quadratically with dimension, whereas \dpmethod handles higher-dimensional settings more effectively due to its greedy approach.

\subsection{Real data}
We evaluate \method, \prevmethod and their DP variants on the Sachs dataset \citep{Sachs}, a well-known benchmark for structure learning with a biologically validated ground-truth DAG. To simulate a federated setting, the data was partitioned across $3$ participants. Experimental details and hyperparameter settings are provided in Appendix \ref{app:real_data}.

Both \method and \prevmethod recover a DAG with $18$ edges, with SHD$=20$ for \method and SHD$=23$ for \prevmethod, consistent with previously reported NOTEARS results ($16$ edges, SHD$=22$). Direct comparison with \prevmethod from original paper is not possible since their experiments were conducted on a smaller dataset.
As we have no guarantee that the ground-truth DAG is identifiable in real datasets, we also assess the Markov equivalence class, by looking at the skeleton (undirected edges) and v-structures. For a detailed introductions to these notions, we refer the reader to \citet{BN}.
Neither method recovers ground-truth v-structures. However, \method correctly identifies $12$ out of $18$ undirected edges, compared to $8$ out of $18$ undirected edges for \prevmethod. 

For the DP variants, we fix the privacy budget to $\varepsilon=5, \delta=\frac{1}{n_p^2}$. Both methods experienced a moderate performance drop: \dpmethod achieved SHD$=21$ while still recovering $11$ correct edges out of $18$, whereas DP-\prevmethod obtained SHD$=25$ with $8$ correct edges out of $19$ estimated.

\section{CONCLUSION}
\label{sec:conclu}

We introduced \method, a federated method for learning linear Gaussian Bayesian networks that addresses privacy and communication challenges. By combining sparse, greedy updates with differential privacy, \method achieves accurate structure recovery with low communication overhead. Experiments on synthetic and real datasets demonstrate its effectiveness, scalability, and support for participant-level personalization. 

Providing formal convergence guarantees for our proposed algorithms is an interesting direction for future work. While non-private PGCD enjoys linear convergence \citep{karimireddy2019Efficient}, extending these guarantees to the differentially private setting introduces significant technical challenges. As noted by \citet{mangold}, deriving convergence guarantees for private greedy coordinate descent algorithms with proximal operators remains an open problem. Moreover, even in the centralized setting and assuming exact solutions to the subproblems, the convergence of the NOTEARS formulation itself is not fully established: the acyclicity constraint violates the regularity conditions typically required for global convergence guarantees in augmented Lagrangian methods \citep{conv_notears}.\looseness=-1

In future work, we also plan to extend our approach to hybrid Bayesian networks to handle data with both continuous and discrete features. 

\newpage

\subsubsection*{Acknowledgements}
We thank the anonymous reviewers for their valuable feedback and constructive comments. 

This work was partially supported by L'Oréal, by grant ANR-20-CE23-0015 (Project PRIDE), and by the ANR 22-PECY-0002 IPOP (Interdisciplinary Project on Privacy) project of the Cybersecurity PEPR.

Some experiments presented in this paper were carried out using the Grid'5000 testbed, supported by a scientific interest group hosted by Inria and including CNRS, RENATER and several Universities as well as other organizations (see \url{https://www.grid5000.fr}).

\bibliography{biblio}




\section*{Checklist}



\begin{enumerate}

  \item For all models and algorithms presented, check if you include:
  \begin{enumerate}
    \item A clear description of the mathematical setting, assumptions, algorithm, and/or model. \textbf{Yes}
    \item An analysis of the properties and complexity (time, space, sample size) of any algorithm. \textbf{No}
    \item (Optional) Anonymized source code, with specification of all dependencies, including external libraries. \textbf{Yes}
  \end{enumerate}

  \item For any theoretical claim, check if you include:
  \begin{enumerate}
    \item Statements of the full set of assumptions of all theoretical results. \textbf{Yes}
    \item Complete proofs of all theoretical results. \textbf{Yes}
    \item Clear explanations of any assumptions. \textbf{Yes}     
  \end{enumerate}

  \item For all figures and tables that present empirical results, check if you include:
  \begin{enumerate}
    \item The code, data, and instructions needed to reproduce the main experimental results (either in the supplemental material or as a URL). \textbf{Yes}
    \item All the training details (e.g., data splits, hyperparameters, how they were chosen). \textbf{Yes}
    \item A clear definition of the specific measure or statistics and error bars (e.g., with respect to the random seed after running experiments multiple times). \textbf{Yes}
    \item A description of the computing infrastructure used. (e.g., type of GPUs, internal cluster, or cloud provider). \textbf{Yes}. 
  \end{enumerate}

  \item If you are using existing assets (e.g., code, data, models) or curating/releasing new assets, check if you include:
  \begin{enumerate}
    \item Citations of the creator If your work uses existing assets. \textbf{Yes}, as detailed in Appendix \ref{app:add-info}.
    
    \item The license information of the assets, if applicable. \textbf{Yes}, as detailed in Appendix \ref{app:add-info}. 
    \item New assets either in the supplemental material or as a URL, if applicable. \textbf{Not Applicable}
    \item Information about consent from data providers/curators. \textbf{Not Applicable}
    \item Discussion of sensible content if applicable, e.g., personally identifiable information or offensive content. \textbf{Not Applicable}
  \end{enumerate}

  \item If you used crowdsourcing or conducted research with human subjects, check if you include:
  \begin{enumerate}
    \item The full text of instructions given to participants and screenshots. \textbf{Not Applicable}
    \item Descriptions of potential participant risks, with links to Institutional Review Board (IRB) approvals if applicable. \textbf{Not Applicable}
    \item The estimated hourly wage paid to participants and the total amount spent on participant compensation. \textbf{Not Applicable}
  \end{enumerate}

\end{enumerate}
\onecolumn
\newpage
\appendix
\section{Theoretical Derivations and Proofs}
\subsection{Reconstruction of Covariance Matrix}
\label{app:cov}
Starting from the closed-form solution of the local subproblem in participant $p$, 
\[
B = (\Sigma + \rho I_d)^{-1} (\rho W - \beta + \Sigma),
\]
we can isolate $\Sigma$ by multiplying both sides by $\Sigma + \rho I_d$ and rearranging terms: 
\[
(\Sigma + \rho I_d) B  = \rho W - \beta + \Sigma \quad \Rightarrow \quad \Sigma (B - I_d) = \rho(W - B) - \beta.
\]

Now, assume for the moment that $B - I_d$ is invertible. Then, the covariance matrix can be written as
\[
\Sigma = (B - I_d)^{-1} \big(\rho (W - B) - \beta \big).
\]

Regarding the invertibility of $B-I_d$, note that at the first iteration where $W=0$ and $\beta=0$, we have
\[
B - I_d = (\Sigma + \rho I_d)^{-1} \Sigma - I_d = -\rho(\Sigma + \rho I_d)^{-1}.
\]
This matrix is invertible because $\Sigma + \rho I_d$ is strictly positive definite: $\Sigma$ is the empirical covariance matrix (positive semidefinite) and $\rho > 0$. Consequently, $B - I_d$ is invertible at the first iteration.\footnote{In practice, we find that this property generally holds also in subsequent iterations.}

The above derivation shows that the server can reconstruct the empirical covariance $\Sigma$ from $B, W, \beta$ and $\rho$ at the first iteration of the algorithm.

\subsection{Privacy Proofs}
\label{app:privacy-proof}
\subsubsection{Preliminaries on Differential Privacy}
We recall the main definitions and results used in our privacy analysis.  

\begin{definition}[\citealt{10.1007/978-3-662-53641-4_24}]
\label{def:zCDP}
A randomized algorithm $\mathcal{A}$ satisfies $\rho$-zero-concentrated differential privacy ($\rho$-zCDP), if for any two datasets $D_1$ and $D_2$ of fixed size that differ in exactly one record  and all $\alpha \in (1, \infty)$, we have :
\[
    D_\alpha(\mathcal{A}(D_1) \| \mathcal{A}(D_2))\leq \rho\alpha
\]
where $D_\alpha(\mathcal{A}(D_1) \| \mathcal{A}(D_2))$ is the $\alpha$-Rényi divergence between the distributions of $\mathcal{A}(D_1)$ and $\mathcal{A}(D_2)$.
\end{definition}

\begin{theorem}[\citealt{10.1007/978-3-662-53641-4_24}]
    \label{thm:compo-zCDP}
    If $\mathcal{A}_1$ satisfies $\rho^{(1)}$-zCDP and $\mathcal{A}_2$ satisfies $\rho^{(2)}$-zCDP, then the composition of $\mathcal{A}_1$ and $\mathcal{A}_2$ satisfies $(\rho^{(1)} + \rho^{(2)})$-zCDP.
\end{theorem}
\begin{theorem}[\citealt{10.1007/978-3-662-53641-4_24}]
\label{thm:zCDP-to-DP}
    If $\mathcal{A}$ satisfies $\rho$-zCDP, then $\mathcal{A}$ satisfies $(\rho + 2\sqrt{\rho\log(1/\delta)}, \delta)$-DP for any $\delta > 0$.
\end{theorem}

\begin{theorem}[\citealt{10.1007/978-3-662-53641-4_24}]
    \label{thm:gauss}
    Let $\mathcal{A}$ be the Gaussian mechanism applied to a query with $\ell_2$-sensitivity $\Delta$, using noise parameter $\sigma$. Then $\mathcal{A}$ satisfies $\rho$-zCDP with $\rho = \frac{\Delta^2}{2 \sigma^2}$.
\end{theorem}

\begin{theorem}[\citealt{exp-zcdp}]
\label{thm:zCDP_gumbel}
    Let $\mathcal{A}$ be the Gumbel max trick mechanism applied to a query with $\ell_1$-sensitivity $\Delta$, using noise parameter $\beta$. Then $\mathcal{A}$ satisfies $\rho$-zCDP with $\rho = \frac{\Delta^2}{2 \beta^2}$.
\end{theorem}

\subsubsection{Proof of Theorem~\ref{thm:fed-dp}}
\begin{proof}
The smooth part of the local objective is
\[
\mathcal{L}(B;X)
= L(B;X) 
+ \operatorname{tr}\!\big[\beta^\top (B-W)\big]
+ \frac{\rho_2}{2}\|B-W\|_F^2,
\]
where only $L(B;X)$ depends on the data:
\[
L(B;X) = \frac{1}{2n}\|X - X B\|_F^2
= \frac{1}{n}\sum_{i=1}^n \ell(B;x_i),
\qquad
\ell(B;x_i) = \frac{1}{2}\|x_i - B^\top x_i\|_2^2.
\]
The $(i,j)$-th partial derivative of $\ell$ is $\nabla_{i,j}\ell(B;x_i)$.

\paragraph{Gradient sensitivity.}
For neighboring datasets $X,X'$ differing in a single sample at index $k$, we have
\[
\Delta
= \sup_{X,X'} 
\|\nabla_{i,j}\mathcal{L}(B;X) - \nabla_{i,j}\mathcal{L}(B;X')\|_2
= \frac{1}{n} \big|\nabla_{i,j}\ell(B;x_k) - \nabla_{i,j}\ell(B;x'_k)\big|
\le \frac{2L_{i,j}}{n},
\]
where $L_{i,j}$ is the coordinate-wise Lipschitz constant of $\ell$.

\paragraph{Sensitivity of the coordinate selection score.}
The non-private score is
\[
S_{i,j} = \sqrt{M_{i,j}}
\Big|
\operatorname{prox}_{\frac{\lambda}{M_{i,j}}|\cdot|}
\Big(B_{i,j}^{k} - \frac{1}{M_{i,j}}\nabla_{i,j}\mathcal{L}(B^{k};X)\Big)
- B_{i,j}^{k}
\Big|.
\]
By the non-expansiveness of the proximal operator, the sensitivity satisfies:
\begin{align}
\Delta(S_{i,j})
&= \Big| \sqrt{M_{i,j}} \Big| 
\operatorname{prox}_{\frac{\lambda}{M_{i,j}}|\cdot|}
\Big( B_{i,j} - \frac{1}{M_{i,j}} \nabla_{i,j} \mathcal{L}(B;X) \Big)
- B_{i,j} \Big| \nonumber\\
&\quad - \sqrt{M_{i,j}} \Big|
\operatorname{prox}_{\frac{\lambda}{M_{i,j}}|\cdot|}
\Big( B_{i,j} - \frac{1}{M_{i,j}} \nabla_{i,j} \mathcal{L}(B;X') \Big)
- B_{i,j} \Big| \Big| \nonumber\\
&\le \sqrt{M_{i,j}} 
\Big| 
\Big(B_{i,j} - \frac{1}{M_{i,j}}\nabla_{i,j} \mathcal{L}(B;X) \Big)
- \Big(B_{i,j} - \frac{1}{M_{i,j}}\nabla_{i,j} \mathcal{L}(B;X') \Big)
\Big| \nonumber\\
&= \frac{1}{\sqrt{M_{i,j}}} 
\big| \nabla_{i,j} \mathcal{L}(B;X) - \nabla_{i,j} \mathcal{L}(B;X') \big|
\le \frac{2L_{i,j}}{n\sqrt{M_{i,j}}} \nonumber\\
&= \frac{\Delta}{\sqrt{M_{i,j}}}.
\label{eq:sensitivity-score}
\end{align}

\paragraph{Base mechanisms.}

We now analyze the privacy guarantees of the two base mechanisms we use:

\begin{enumerate}
    \item \textbf{Coordinate selection via Gumbel-max (Exponential mechanism).}  
    The non-private score $S_{i,j}$ from \eqref{eq:sensitivity-score} has sensitivity $\Delta(S_{i,j})$. From Theorem \ref{thm:zCDP_gumbel}, adding Gumbel noise $\mathrm{Gumbel}(0,\beta_{i,j})$ to this score ensures $\rho_\text{EM}$-zCDP with
\[
    \rho_\text{EM} = \frac{\Delta(S_{i,j})^2}{2\beta_{i,j}^2} = \frac{1}{M_{i,j}}\frac{\Delta^2}{2\beta_{i,j}^2}
\]

    \item \textbf{Gradient update via Gaussian mechanism.}  
    Each gradient coordinate $\nabla_{i,j}\mathcal{L}(B^k; X)$ has $\ell_2$-sensitivity $\Delta$. Adding Gaussian noise $\mathcal{N}(0, \sigma_{i,j}^2)$ to this coordinate ensures $\rho_\text{Gauss}$-zCDP (Theorem \ref{thm:gauss}) with
    \[
        \rho_\text{Gauss} = \frac{\Delta^2}{2 \sigma_{i,j}^2}.
    \]

    \item \textbf{Matching privacy costs.}  
    To balance the privacy budget between the two mechanisms, we set
    \[
        \beta_{i,j} = \frac{\sigma_{i,j}}{\sqrt{M_{i,j}}}.
    \]
    With this choice, both mechanisms contribute equally to the total zCDP cost:
    \[
        \rho = \frac{\Delta^2}{2\sigma_{i,j}^2}.
    \]
\end{enumerate}

\paragraph{Composition and DP conversion.}
Since there are $2KT$ queries in total, by the composition theorem (Theorem~\ref{thm:compo-zCDP}),
\[
\rho_\text{total} = 2KT \rho = \frac{KT \Delta^2}{\sigma_{i,j}^2}.
\]
Applying the zCDP to $(\varepsilon,\delta)$-DP conversion (Theorem~\ref{thm:zCDP-to-DP}) yields
\[
\varepsilon = \rho_\text{total} + 2 \sqrt{\rho_\text{total} \log(1/\delta)}.
\]
Therefore, to satisfy $(\varepsilon,\delta)$-DP, we need to set the noise scales to
\[
\beta_{i,j} = \frac{\sigma_{i,j}}{\sqrt{M_{i,j}}} 
= \frac{\Delta \sqrt{KT}(\sqrt{\log(\frac{1}{\delta}) + \varepsilon} +\sqrt{\log(\frac{1}{\delta})})}{\varepsilon}.\qedhere
\]
\end{proof}

\subsection{Smoothness Constants}
\label{app:smoothness_theo}

In our method, we assume the coordinate-wise smoothness constants $M_{i,j}$ known or upper bounded. When it is not the case, these constants can be estimated privately. 
\subsubsection{Estimation of the Constants}
\textit{Coordinate-wise smoothness.} A differentiable function $f : \mathbb{R}^{d \times d} \rightarrow \mathbb{R}$ is called \emph{coordinate-wise smooth} if there exists $M \in \mathbb{R}^{d \times d}$ such that for all $X \in \mathbb{R}^{d \times d}$ and for all $1 \leq i,j \leq d$, 
\[
    \lvert \nabla_{i,j} f(X+uE_{i,j}) - \nabla_{i,j}f(X) \rvert \leq M_{i,j} \lvert u \rvert, \hspace{0.5cm} \forall u \in \mathbb{R},
\]
where $E_{i,j}$ is the matrix with 1 in entry $(i,j)$ and $0$ elsewhere.

Consider the local objective: 
\[
\mathcal{L}(B;X)
= L(B;X) 
+ \operatorname{tr}\!\big[\beta^\top (B-W)\big]
+ \frac{\rho_2}{2}\|B-W\|_F^2,
\]
where $X \in \mathbb{R}^{n \times d}$ is the local data matrix.
The Hessian with respect to $B$ is 
\[
\nabla^2\mathcal{L}(B;X) = \frac{1}{n} X^\top X \otimes I_d + \rho_2 I_{d^2}.
\]
Hence, the coordinate-wise smoothness constant for coordinate $(i,j)$ is:
\begin{align}
    M_{i,j}(X) = M_j(X) &= \frac{1}{n} \|X_{:,j}\|_2^2 + \rho_2 \qquad \forall i \in {1, \dots, d}\nonumber \\
    & = \frac{1}{n} \sum_{k=1}^n X_{k,j}^2 + \rho_2\label{eq:our_smooth_coord}
\end{align}
As can be seen from the formula above, $M_{i,j}(X)$ does not depend on $i$, hence the notation $M_j(X)$.

\subsubsection{Differentially Private Estimation}

We now explain how to privately estimate the coordinate-wise smoothness constants. Following \citet{mangold}, we assume we know an upper bound $b_j$ on the squared value of the $j$-th feature, i.e., $X_{j}^2 \leq b_j$.
    
\begin{theorem}[Privacy of smoothness constants]
\label{thm:smoothness_cst}
    Let $\varepsilon_M, \delta_M > 0$ and $\Delta_{M,j}=b_j/n_p$ where $b_j$ is such that $X_{j}^2 \leq b_j$ and $n_p$ is the sample size of participant $p$.
    If the Gaussian noise parameter is chosen as
\[
\sigma_j =  \frac{\Delta_{M,j} \sqrt{d} (\sqrt{\log(1/\delta_M) + \varepsilon_M} + \sqrt{\log(1/\delta_M)})} {\sqrt{2}\varepsilon_M},\quad\forall j \in \{1, \dots, d\},
\]
then participant $p$ can locally estimate all $M_{i,j}$'s with $(\varepsilon_M, \delta_M)$-DP in a preprocessing step using $d$ calls to the Gaussian mechanism with noise scales $\sigma_1,\dots,\sigma_d$.
\end{theorem}
\begin{proof}
As shown in \eqref{eq:our_smooth_coord}, we only need to estimate $M_1(X),\dots,M_d(X)$ to obtain all coordinate-wise smoothness constants. 
 Let $X$ and $X'$ two neighboring datasets that differ in exactly one datapoint. The sensitivity of $M_{j}$ is
\[
\Delta_{M,j} = \sup_{X,X'} \|  M_{j}(X) -  M_{j}(X') \|_2 = \frac{1}{n_p} \lvert X_{k,j}^2 - X_{k,j}^{'2}\rvert \leq \frac{b_j}{n_p}.
\]
To privatize $M_{j}$, we clip and add Gaussian noise: 
\[
\tilde{M}_{j}(X) = \frac{1}{n_p} \sum_{k=1}^{n_p} \min(X_{k,j}^2, b_j) + \rho_2 + \mathcal{N}(0,\sigma_j^2).
\]

By Theorem \ref{thm:gauss}, releasing $\tilde{M}_{j}(X)$ satisfies $\rho^{(j)}$-zCDP with
    \[
        \rho^{(j)} = \frac{\Delta_{M,j}^2}{2 \sigma_{j}^2}.
    \] 
Since there are $d$ constants to estimate privately, by the composition theorem (Theorem \ref{thm:compo-zCDP}),
    \[
    \rho_{\text{total}} = d\frac{\Delta_{M,j}^2}{2 \sigma_{j}^2}.
    \]

Applying the zCDP to $(\varepsilon_M,\delta_M)$-DP conversion (Theorem~\ref{thm:zCDP-to-DP}) yields
\[
\varepsilon_M = \rho_\text{total} + 2 \sqrt{\rho_\text{total} \log(1/\delta)}.
\]

Therefore, to satisfy $(\varepsilon_M,\delta_M)$-DP, we need to set the noise scale to
\[
\sigma_j =  \frac{\Delta_M \sqrt{d} (\sqrt{\log(1/\delta_M) + \varepsilon_M} + \sqrt{\log(1/\delta_M)})} {\sqrt{2}\varepsilon_M}.\qedhere
\]
\end{proof}

\textbf{Remark.} The above theorem shows how each participant $p$ can estimate the coordinate-wise smoothness constants based on its local data. These participant-level constants can then be shared and aggregated to obtain more accurate global estimates.

\subsection{Theoretical Comparison of DP-PGCD Privacy Bounds}
\label{app:DPPGCD-comparison}

We provide a detailed mathematical comparison of the privacy guarantees between our improved DP-PGCD algorithm (Algorithm~\ref{algo:DP-PGCD}) and the original approach introduced by \citet{mangold}, which relied on Laplace noise. The following analysis highlights why our method requires less noise to achieve the same privacy budget.

The privacy loss for our approach is given by:
\begin{equation*}
    \varepsilon_{\text{new}} = \frac{4 L_{ij} \sqrt{KT \log(1/\delta)}}{n \sigma_{ij}} + \frac{4 L_{ij}^2 KT}{n^2\sigma_{ij}^2},
\end{equation*}
while the original approach of \citet{mangold}, it is:
\begin{equation*}
    \varepsilon_{\text{old}} = \frac{4 L_{ij} \sqrt{KT \log(1/\delta)}}{n \lambda_{ij}} + \frac{4 L_{ij} KT}{n\lambda_{ij}} \Big(e^{\frac{2L_{ij}}{n\lambda_{ij}}} - 1\Big),
\end{equation*}
where $\lambda_{ij}$ is the scale parameter of a Laplace distribution.

To compare these expressions, we match the noise variance: for a Gaussian distribution $\mathcal{N}(0, \sigma^2)$, the variance is $\sigma^2$, while for a Laplace distribution $\text{Lap}(0,\lambda)$, it is $2\lambda^2$. Setting $\sigma^2 = 2\lambda^2$ (i.e., $\lambda = \frac{\sigma}{\sqrt{2}}$) yields:
\begin{equation*}
    \varepsilon_{\text{new}} = \frac{4 L_{ij} \sqrt{KT \log(1/\delta)}}{n \sigma_{ij}} + \frac{4 L_{ij}^2 KT}{n^2\sigma_{ij}^2},
\end{equation*}
\begin{equation*}
    \varepsilon_{\text{old}} = \frac{4\sqrt{2} L_{ij} \sqrt{KT \log(1/\delta)}}{n \sigma_{ij}} + \frac{4\sqrt{2} L_{ij} KT}{n\sigma_{ij}} \Big(e^{\frac{2\sqrt{2}L_{ij}}{n\sigma_{ij}}} - 1\Big).
\end{equation*}
We analyze the ratio $\frac{\varepsilon_{\text{old}}}{\varepsilon_{\text{new}}}$:
\begin{align*}
    \frac{\varepsilon_{\text{old}}}{\varepsilon_{\text{new}}} & =  \sqrt{2} \frac{1 + \sqrt{\frac{KT}{\log(1/\delta)}} \Big(e^{2\sqrt{2}\frac{L_{ij}}{n\sigma_{ij}}} - 1\Big)}{1 + \sqrt{\frac{KT}{\log(1/\delta)}} \frac{L_{ij}}{n\sigma_{ij}}}.
\end{align*}
Using the inequality $e^x \geq 1 + x$ for all $x \in \mathbb{R}$, we have: 
\begin{align*}
   & e^{2\sqrt{2}\frac{L_{ij}}{n\sigma_{ij}}} - 1 \geq 2\sqrt{2}\frac{L_{ij}}{n\sigma_{ij}} > \frac{L_{ij}}{n\sigma_{ij}}.
\end{align*}
By multiplying both sides by $\sqrt{\frac{KT}{\log(1/\delta)}}$ and adding 1, we have: 
\begin{align*}
     & 1 + \sqrt{\frac{KT}{\log(1/\delta)}} (e^{2\sqrt{2}\frac{L_{ij}}{n\sigma_{ij}}} - 1) > 1 + \sqrt{\frac{KT}{\log(1/\delta)}}\frac{L_{ij}}{n\sigma_{ij}}. 
\end{align*}  
Equivalently,
\begin{align*}
     & \frac{1 + \sqrt{\frac{KT}{\log(1/\delta)}} (e^{2\sqrt{2}\frac{L_{ij}}{n\sigma_{ij}}} - 1)}{1 + \sqrt{\frac{KT}{\log(1/\delta)}} \frac{L_{ij}}{n\sigma_{ij}}} > 1.
\end{align*}  
Finally, by multiplying both sides by $\sqrt{2}$, we obtain:
\begin{align*}  
    &  \frac{\varepsilon_{\text{old}}}{\varepsilon_{\text{new}}} \geq \sqrt{2}.
\end{align*}  
Thus, our DP-PGCD (Algorithm~\ref{algo:DP-PGCD}) provides at least a $\sqrt{2}$ improvement in privacy loss compared to \citet{mangold}, for the same noise variance.

\section{Experimental Setup and Implementation}
\subsection{Datasets}
\subsubsection{Synthetic Data: Homogeneous Setting}
\label{app:data_homo}

To simulate homogeneous data, we reproduce and generalization the data generation procedure used in \citep{notears,fed-notears}. We generate a dataset for $P$ participants, each having $n_p$ samples, across $d$ variables. This setup assumes a shared underlying structure and identical edge weights across all participants. 

\begin{itemize}
    \item \textbf{DAG structure generation:} A $d \times d$ binary adjacency matrix is initially created. This is achieved by generating an Erdös-Rényi random graph with $d$ nodes and an expected number of $d$ edges. The generation process ensures the resulting graph is acyclic.
    \item \textbf{Edge weight assignment:} Once the DAG structure (represented by the binary adjacency matrix) is defined, non-zero weights are assigned to its edges. These weights are uniformly sampled from two disjoint intervals: $[-2, -0.5] \cup [0.5, 2]$. This results in a weighted adjacency matrix.
    \item \textbf{Data generation:} Observations are then simulated from the linear Gaussian Structural Equation Model (SEM) defined in  Section~\ref{sec:pb-setting}. In this model, the value of each variable is determined by a linear combination of its parents' value (as defined by the weighted adjacency matrix), plus an independent Gaussian noise term. This additive noise term is sampled, for each variable, from a standard normal distribution $\mathcal{N}(0,1)$. This implies a uniform noise variance of 1 across all variables, a condition that ensures the identifiability of the DAG structure \citep{identif,loh}.
\end{itemize}

\subsubsection{Synthetic Data: Heterogeneous Setting}
\label{app:data_hetero}

For heterogeneous data generation, we consider $P=5$ participants, each having $n_p=5000$ samples, and we assume that all $P$ participants share the same underlying DAG structure but differ in their edge weights. Instead of being identical for all participants, we use a hierarchical design: each participant's weight for an edge present in the DAG is independently sampled from a Gaussian distribution centered at the corresponding global weight, with variance $\sigma^2 = 0.1$. Global weights are uniformly drawn from the disjoint intervals $[-2, -0.5] \cup [0.5, 2]$ as in the homogeneous setting. Observations for each participant are then generated from a linear Gaussian SEM using their individual edge weights.

\subsubsection{Real Data}
\label{app:real_data}
The Sachs dataset \citep{Sachs}, composed of $d=11$ variables representing protein signaling molecules and $n=7466$ samples, features a biologically validated ground-truth DAG with $18$ edges.
To simulate a federated setting, the dataset was partitioned across $P=3$ participants, resulting in $n_p = 2488$ samples per participant. 

\subsection{Evaluation Metrics}
\label{app:metrics}
\subsubsection{Communication cost}
The total communication cost is defined as the cumulative volume of data exchanged between the server and all participants over the entire training procedure, reported in megabytes (MB).
Each floating-point value is encoded using $8$ bytes.

When sparsity is used, the transmission of each non-zero coefficient requires sending its matrix index, encoded using $\lceil \log_2(d^2) / 8 \rceil$ bytes.
Hence, the total size of one transmitted coefficient-index pair is $b_{\text{entry}} = 8 + \lceil \log_2(d^2) / 8 \rceil$ bytes.

In \prevmethod, a dense $d \times d$ matrix is transmitted from each of the $P$ participants to the server and back at every iteration, resulting in a total communication cost accumulated over $T$ iterations of 
\[
\text{Cost} = 2 \times T \times P \times d^2 \times 8 \text{bytes}.
\]
Since all entries are transmitted, no index encoding is required in this case.

In contrast, \method communicates only non-zero coefficients and their indices are transmitted from each of the $P$ participants to the server and back at every iteration, resulting in a total communication cost accumulated over $T$ iterations of 
\[
\text{Cost} = \sum_{t=1}^T \Big(
\underbrace{\sum_{p=1}^P e_{\text{local}}^{(p,t)} \times b_{\text{entry}}}_{\text{participants}}
+ 
\underbrace{P \times e_{\text{global}}^{(t)} \times b_{\text{entry}}}_{\text{server}}
\Big),
\]
where $e_{\text{local}}^{(p,t)}$ denotes the number of coefficients sent by participant $p$ at iteration $t$, and $e_{\text{global}}^{(t)}$ denotes the number of global coefficients broadcast by the server. 

\subsubsection{Personalization}
In the heterogeneous setting, each participant $p$ has its own regression coefficients $B_{\text{true}}^{(p)}$ within a shared DAG structure.
After learning the consensus structure, \method returns a global weighted adjacency matrix $\widehat{B}_\text{global}$. Then, each participant locally re-estimates its coefficients $\widehat{B}^{(p)}$ by ordinary least squares regression on its own data, keeping the DAG fixed.

The quality of estimated parameters is measured by the normalized mean squared error (MSE): 
\[
\mathrm{MSE}_p(\widehat{B}) = \frac{\| \widehat{B} - B_{\text{true}}^{(p)} \|_F^2}{\|B_{\text{true}}^{(p)}\|_F^2}.
\]
In Figure \ref{fig:MSE_fed_solo}, $\mathrm{MSE}_p(\widehat{B}^{(p)})$ is \method with personalization and $\mathrm{MSE}_p(\widehat{B}_\text{global})$ is \method without personalization.

\subsection{Hyperparameters}
\label{app:hyperparams}
The selection of appropriate hyperparameters is critical for the performance and fair comparison of algorithms. This appendix details the hyperparameter tuning methodology employed for \method, \dpmethod, \prevmethod and \dpprevmethod across all experimental settings: homogeneous and heterogeneous synthetic data, real data, across both non-private and private settings. Our general goal was to ensure that each method operated as its optimal performance for each specific scenario. 

\subsubsection{Synthetic Data: General Tuning Protocol}

For all settings, the metric used to tune the hyperparameters is the Structural Hamming Distance (SHD). 

\textbf{Tuning on held-out problem instances.} For experiments with dimension $d=20$, hyperparameters for each algorithm were tuned on an independent dataset generated with seed$=1$. For higher-dimensional experiments ($d \in \{50,100,200\}$), hyperparameters were tuned based on the average performance across two independent datasets, generated with seed$=100$ and seed$=400$. This approach helps to ensure robustness of the chosen hyperparameters against specific dataset realizations in high dimensions.

\textbf{Final evaluation.} Once the optimal hyperparameters were determined for a given scenario using the procedure above, the final results (mean and standard deviation) reported in Section \ref{sec:expe} were obtained by running each algorithm on $10$ new datasets generated with $10$ independent seeds: $\{2, \dots, 11\}$. 

\subsubsection{Synthetic Data: Non-Private Setting}

In the non-private setting, we tuned the ADMM penalty parameters $(\rho_1, \rho_2)$, the step size for gradient descent in \method ($\gamma$), and the regularization parameter ($\lambda$). As done in prior work \citep{notears,fed-notears}, we also use a threshold set at $0.3$ for edge pruning in a post-processing step.

\paragraph{Homogeneous data.}
For homogeneous synthetic data, the grid search ranges for each hyperparameter are detailed in Table \ref{tab:homogeneous}.
\begin{table}
\centering
\caption{Grid search for hyperparameters tuning in the non-private homogeneous setting.}
\label{tab:homogeneous}
\begin{tabular}{ll}
\toprule
\textbf{Hyperparameter} & \textbf{Grid search} \\
\midrule
$\rho_1$ &  $\{10,50,100,1000,10000\}$\\
$\rho_2$ & $\{1,10,100,1000\}$ \\
$\gamma$ & $\{0.1, 0.5, 1\}$ \\
$\lambda$ & $\{0.001, 0.01, 0.1, 0.5, 1\}$ \\
\bottomrule
\end{tabular}
\end{table}

The best configurations of hyperparameters, selected based on the lowest SHD on the tuning dataset (seed$=1$) are:
\begin{itemize}
    \item For \method: $\rho_1=1000$, $\rho_2=1$ $\lambda=0.1$ and $\gamma=0.5$ 
    \item For \prevmethod: $\rho_1=1000$, $\rho_2=1$ $\lambda=0.01$
\end{itemize}

\paragraph{Heterogeneous data.}
For heterogeneous synthetic data, the grid search ranges were slightly adjusted. The grid search ranges are detailed in Table \ref{tab:heterogeneous}.
\begin{table}
\centering
\caption{Grid search for hyperparameters tuning in the non-private heterogeneous setting.}
\label{tab:heterogeneous}
\begin{tabular}{ll}
\toprule
\textbf{Hyperparameter} & \textbf{Grid search} \\
\midrule
$\rho_1$ &  $\{100,1000,10000\}$\\
$\rho_2$ & $\{1,10,100\}$ \\
$\gamma$ & $\{0.5, 1\}$ \\
$\lambda$ & $\{0.01, 0.1, 1\}$ \\

\bottomrule
\end{tabular}
\end{table}

The best configurations of hyperparameters, selected based on the lowest SHD on the tuning dataset (seed$=1$) are:
\begin{itemize}
    \item For \method: $\rho_1=1000$, $\rho_2=1$ $\lambda=0.1$ and $\gamma=0.5$ 
    \item For \prevmethod: $\rho_1=100$, $\rho_2=1$ $\lambda=0.1$
\end{itemize}

\subsubsection{Synthetic Data: Private Setting}
In the private setting, we focused on tuning parameters specific to differentially private mechanisms. 

\paragraph{Privacy-utility study (Figure \ref{fig:dp_eps}).}
For \dpmethod, we fixed the non-private hyperparameters ($\rho_1, \rho_2, \gamma, \lambda$) to their best-performing values identified in the non-private setting. This allowed us to isolate the impact of privacy-specific parameters. 
We then tuned the following privacy-specific hyperparameters: $C$ is the clipping threshold for gradients in \dpmethod, $T$ the number of ADMM iterations and $K$ the number of local PGCD iterations performed by each participant per ADMM iteration. 

The grid search ranges are given in Table \ref{tab:pu}.
\begin{table}[!ht]
\centering
\caption{Grid search ranges for privacy-specific hyperparameters tuning in the private homogeneous setting for the privacy-utility study (Figure \ref{fig:dp_eps}).}
\label{tab:pu}
\begin{tabular}{ll}
\toprule
\textbf{Hyperparameter} & \textbf{Grid search} \\
\midrule
$C$ & $\{3,5,7,10,20,30\}$ \\
$T$ & $\{10, 20, 50, 100\}$ \\
$K$ & $\{10, 20, 30, 40, 50, 100\}$ \\
\bottomrule
\end{tabular}
\end{table}

The optimal combination of these hyperparameters for each privacy budget $\varepsilon \in \{0.5,1,2, 5,10,25,50\}$ are presented Table \ref{tab:privacy_hyperparams}.
\begin{table}[!ht]
\centering
\caption{Optimal hyperparameters for \dpmethod\ in the privacy-utility study (Figure \ref{fig:dp_eps}), for each privacy budget \(\varepsilon\).}
\label{tab:privacy_hyperparams}
\begin{tabular}{cccc}
\toprule
\(\varepsilon\) & \(C\) & \(T\) & \(K\) \\
\midrule
0.5 & 10 & 10 & 10 \\
1   & 10 & 10 & 10 \\
2   & 5  & 100 & 10 \\
5   & 5  & 100 & 20 \\
25  & 7  & 100 & 30 \\
50  & 5  & 100 & 50 \\
\bottomrule
\end{tabular}
\end{table}

\paragraph{Dimensionality robustness study (Figure~\ref{fig:DP_dim}).}
This study evaluated the performance of \dpmethod compared to the baseline \dpprevmethod as the data dimension $d$ increased, under a fixed privacy budget of $\varepsilon=10$. For this analysis, a broader set of hyperparameters, including 
$\lambda$ and $\gamma$, were re-tuned for each dimension. For \dpprevmethod, $b$ is the sensitivity bound used for the Gaussian mechanism applied to privatize the covariance matrices, as described by \cite{DBLP:conf/uai/Wang18}.
The grid search ranges for this study are given in Table \ref{tab:dim}.
\begin{table}[!ht]
\centering
\caption{Grid search ranges for hyperparameters tuning in the private dimensionality robustness study (Figure~\ref{fig:DP_dim}).}
\label{tab:dim}
\begin{tabular}{ll}
\toprule
\textbf{Hyperparameter} & \textbf{Grid search} \\
\midrule
$C$ & $\{5,10,20,30\}$ \\
$b$ & $\{5,10,20,30\}$ \\
$T$ & $\{50, 100\}$ \\
$K$ & $\{20, 50, 100\}$ \\
$\lambda$ & $\{0.01, 0.1, 1\}$ \\
$\gamma$ & $\{0.5,1\}$\\
\bottomrule
\end{tabular}
\end{table}

The optimal hyperparameters for each method and dimension in the dimensionality robustness study are presented in Table \ref{tab:dim_hyperparams}.
\begin{table}[!ht]
\centering
\caption{Optimal hyperparameters for each method and dimension, with privacy budget \(\varepsilon = 10\), in the private dimensionality robustness study (Figure~\ref{fig:DP_dim}).}
\label{tab:dim_hyperparams}
\begin{tabular}{lcccccccc}
\toprule
Dimension & \multicolumn{5}{c}{\dpmethod} & \multicolumn{3}{c}{\dpprevmethod} \\
\cmidrule(lr){2-6} \cmidrule(lr){7-9}
& \(C\) & \(T\)  & \(K\) & \(\lambda\) & \(\gamma\) & \(b\) & \(T\) & \(\lambda\) \\
\midrule
20 & 10 & 100 & 30 & 0.1 & 0.5 & 7 & 300 & 0.01 \\
50 & 5 & 50 & 50 & 0.1 & 1 & 10 & 300 & 0.1 \\
100 & 30 & 100 & 50 & 0.1 & 1 & 20 & 300 & 0.01 \\
200 & 30 & 100 & 100 & 0.1 & 1 & 10 & 300 & 0.01 \\
\bottomrule
\end{tabular}
\end{table}

\subsubsection{Real Data: Non-Private Setting}

In contrast to synthetic data, where we can select hyperparameters on held-out problem instances, it is not clear how to implement a similar validation procedure on the Sachs dataset. Consistent with prior work \citep{notears,fed-notears}, we select the best hyperparameters for each method based on the final structure learning metrics.

We recall that the hyperparameters in the non-private settings are the ADMM penalty parameters $(\rho_1, \rho_2)$,  the $\ell_1$ regularization parameter ($\lambda$), the step size for gradient descent in \method ($\gamma$). For this dataset, we also tuned the threshold used for edge pruning in a post-processing step.

For the non-private evaluation on the Sachs dataset, the goal of hyperparameter tuning was to identify configurations that yielded the best structural accuracy. Given the known ground-truth DAG has $18$ edges, we specifically looked for configurations that estimated approximately $18$ edges while maximizing the number of correctly identified undirected edges.

The tuning process involved a grid search over the hyperparameters for the two methods. 
The grid search for \method and \prevmethod included:
\begin{table}[!ht]
\centering
\begin{tabular}{ll}
\toprule
\textbf{Hyperparameter} & \textbf{Grid searchs} \\
\midrule
$\rho_1$ &  $\{100,1000,10000, 100000\}$\\
$\rho_2$ & $\{1,5,10,100\}$ \\
$\gamma$ & $\{0.5, 1\}$ \\
$\lambda$ & $\{0.01, 0.1, 1\}$ \\
threshold & $\{0, 0.1, 0.3\}$\\
\bottomrule
\end{tabular}
\end{table}


Note that, wile $\lambda, \rho_1, \rho_2$ and threshold are shared across both methods, $\gamma$ (step size for local updates) is a hyperparameter unique to \method.

\paragraph{Tuning metrics.} During tuning, we looked at several metrics, including the number of estimated edges, the number of correctly estimated undirected edges, the number of estimated v-structures, the number of correctly estimated v-structures and SHD.

Consistent with prior work, we observed that recovering ground-truth v-structures was particularly challenging for both methods on this dataset ($0$ correct v-structures found). Therefore, we prioritized configurations that resulted in an estimated number of edges close to the true $18$ edges of the Sachs DAG, while maximizing correctly identified undirected edges.

The best configurations of hyperparameters are:
\begin{itemize}
    \item For \method: $\rho_1=10000$, $\rho_2=5$ $\lambda=1, \gamma=0.1$ and threshold$=0.1$.
    \item For \prevmethod: $\rho_1=100000$, $\rho_2=10$ $\lambda=0.1$ and threshold$=0.1$.
\end{itemize}

\subsubsection{Real Data: Private Setting}

We recall that the privacy-specific hyperparameters are $C$ the clipping threshold for gradients in \dpmethod, $b$ the sensitivity bound used for the Gaussian mechanism applied to privatize the covariance matrices in \dpprevmethod, $T$ the number of ADMM iterations and $K$ the number of local PGCD iterations performed by each participant per ADMM iteration in \dpmethod. 

For the differentially private evaluation on the Sachs dataset, a fixed privacy budget of $\varepsilon=5$ (with $\delta = 1/n_p^2$) was used for both \dpmethod and \dpprevmethod. To reduce the tuning complexity and build upon the non-private baselines, the non-private hyperparameters ($\rho_1, \rho_2, \lambda, \gamma$) were fixed to their optimal values found in the non-private setting. The tuning then focused on privacy-specific parameters. 
\paragraph{Tuned privacy-specific hyperparameters.} The grids used for privacy-specific parameters for each algorithm are given in Tables~\ref{tab:grid-private-ours}~and~\ref{tab:grid-private-notears}. 

\begin{table}[!ht]
\centering
\caption{Grid search ranges for hyperparameters tuning in the private setting for \dpmethod}
\label{tab:grid-private-ours}
\begin{tabular}{ll}
\toprule
\textbf{Hyperparameter} & \textbf{Grid searchs} \\
\midrule
$C$ & $\{1,10,100,10000,15000,20000\}$ \\
$T$ & $\{30,50\}$ \\
$K$ & $\{10, 20, 30, 50\}$ \\
threshold & $\{0, 0.1, 0.3\}$ \\
\bottomrule
\end{tabular}
\end{table}
\begin{table}[!ht]
\centering
\caption{Grid search ranges for hyperparameters tuning in the private setting for \dpprevmethod}
\label{tab:grid-private-notears}
\begin{tabular}{ll}
\toprule
\textbf{Hyperparameter} & \textbf{Grid searchs} \\
\midrule
$b$ & $\{1,2,3,4,5,100,10000\}$ \\
$T$ & $\{100,300\}$ \\
threshold & $\{0, 0.1, 0.3\}$ \\

\bottomrule
\end{tabular}
\end{table}

The goal was to find the combination of privacy-specific hyperparameters that maintained the best possible utility (low SHD, high correct edges) while operating under the fixed privacy budget $\varepsilon=5$. We aimed for configurations whose performance was close to the non-private baseline.

The selected optimal private configurations are: 
\begin{itemize}
    \item For \dpmethod: $C=15000$, $T=30$ $K=30$, $\lambda=1$, $\gamma=0.1$ and threshold$=0$.
    \item For \dpprevmethod: $b=4, T=100, \lambda=0.1$ and threshold$=0.1$.
\end{itemize}

\subsection{Implementation and Computing Resources}
\label{app:add-info}

All experiments were conducted on CPUs, either on a personal computer or using the Grid5000 cluster (\url{https://www.grid5000.fr/w/Grid5000:Home}).

The implementation relies on several open-source components. Parts of the code were adapted from the open-source implementation of \citet{fed-notears} (available at \url{https://github.com/ignavierng/notears-admm/tree/master/notears\_admm}), which itself is based on the original NOTEARS implementation \citep{notears} (available at \url{https://github.com/xunzheng/notears/tree/master/notears}).
Additionally, the code for the privatization of the covariance matrix for \dpprevmethod was adapted from the repository at \url{https://github.com/BorjaBalle/analytic-gaussian-mechanism}. 

The Sachs dataset \citep{Sachs} was obtained through the Causal Discovery Toolbox (CDT) Python package \citep{Kalainathan2019CDT}.

Regarding licenses, the reused code for NOTEARS-ADMM \citep{fed-notears} and NOTEARS \citep{notears} and the analytic-gaussian-mechanism are all released under the Apache License. The Causal Discovery Toolbox (CTD) is released under the MIT License.

\section{Additional Empirical Evaluations}

\subsection{Impact of Data Standardization on Identifiability}
\label{app:std}

\begin{figure*}[t]
    \centering  
    \begin{subfigure}[t]{1\textwidth}
    \centering
    \includegraphics[width=0.8\textwidth]{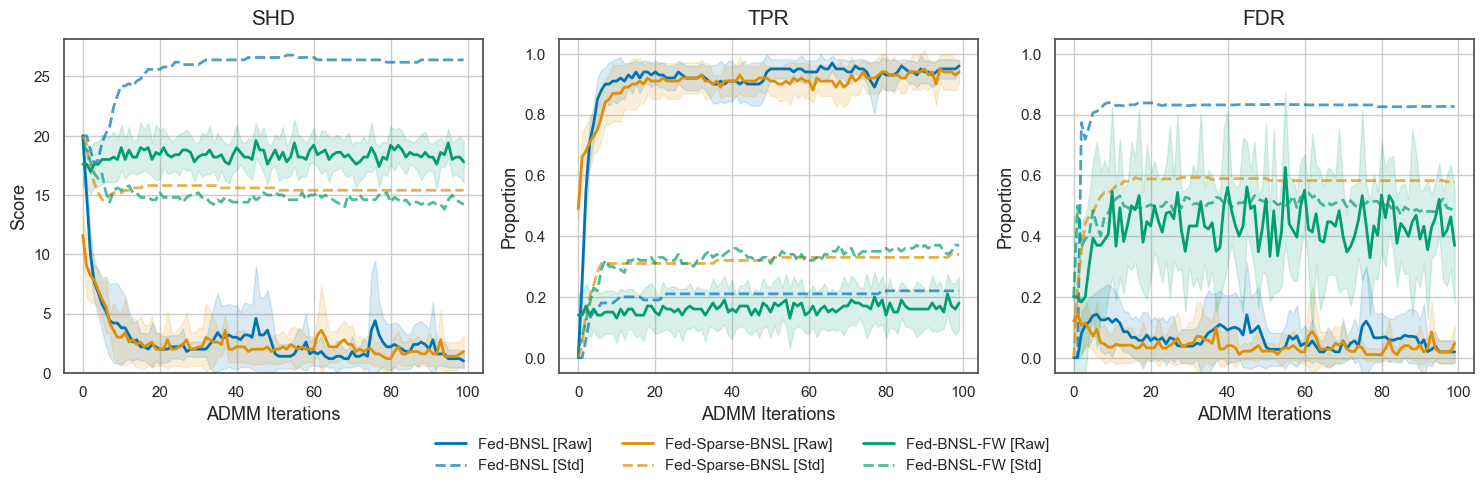} 
    \label{fig:std_cv}
    \end{subfigure}
    \begin{subfigure}[t]{1\textwidth}
    \centering
    \vspace*{-.3cm}
    \includegraphics[width=0.8\textwidth]{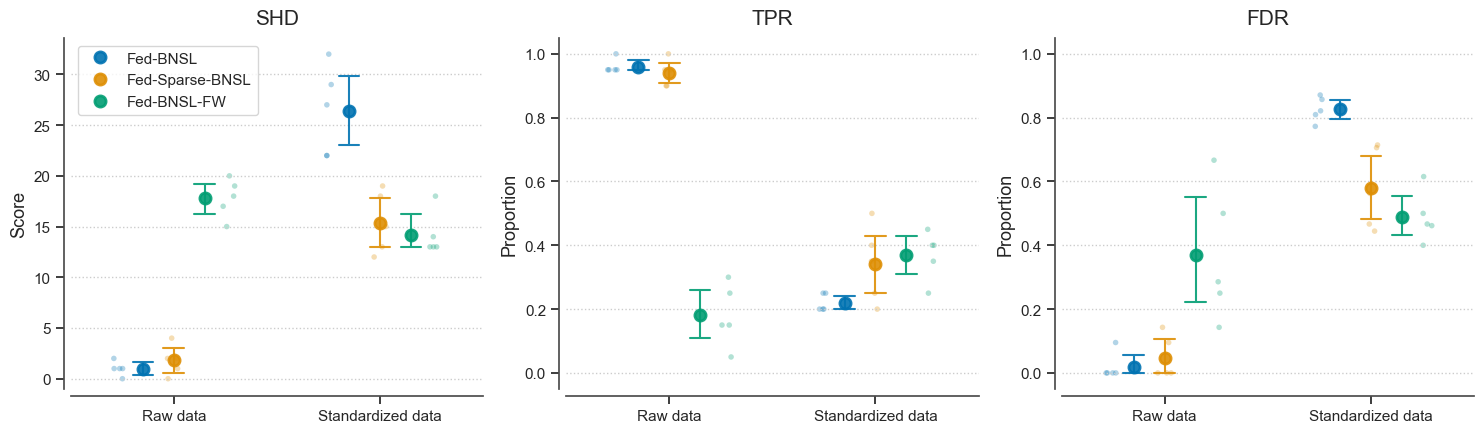} 
    \label{fig:std_final}
    \end{subfigure}
    \caption{Convergence of \prevmethod,  \prevmethod-FW and \method on raw and standardized data. Top: SHD, TPR and FDR across iterations. Solid lines represent raw data, while dashed lines represent standardized data.; Bottom: Final SHD, TPR and FDR (iteration 100).}
    \label{fig:std}
\end{figure*}

In Section~\ref{sec:local_solver}, we theoretically justified the choice of Proximal Greedy Coordinate Descent (PGCD) over classical LASSO solvers (such as Frank-Wolfe). Classical solvers typically require standardized data to apply the $\ell_1$ penalty uniformly. However, in the setting of linear Gaussian BNs, standardizing the data violates the equal error variance assumption, thereby destroying the identifiability of the true DAG \citep{ng,loh}. 
To empirically validate this crucial design choice, we evaluate the performance of: 
\begin{itemize}
    \item Fed-BNSL: the baseline \citep{fed-notears} that solves the unpenalized local subproblem using the exact closed-form solution (Eq.~\ref{eq:closed-form});
    \item \method: solves $\ell_1$-penalized subproblem using PGCD;
    \item Fed-BNSL-FW: a variant of our \method, which solves the $\ell_1$-penalized subproblem using Frank-Wolfe \citep{FW}, a classical LASSO solver.
\end{itemize}
We run each algorithm on both raw data, and standardized data across 5 runs. The results are presented in Figure~\ref{fig:std}, showing the optimization trajectories (top) and the final structural utility (bottom). 
The results confirm that standardization destroys the identifiability of the DAG: when the data is standardized, all three algorithms, including the unpenalized baseline, fail to converge. 
On raw data, identifiability is preserved, and the closed-form baseline (\prevmethod) recovers the true graph. However, for Fed-BNSL-FW which solves the $\ell_1$-penalized problem using Frank-Wolfe, the lack of standardization prevents the solver from applying the $\ell_1$ penalty uniformly, which leads to a highly inaccurate DAG.
In contrast, our proposed approach, \method, which employs PGCD as the solver, effectively resolves this dilemma. The coordinate-wise smoothness constants $M_j$ inherently normalize the gradient steps, enabling PGCD to handle the $\ell_1$-penalized problem on unstandardized data. As a result, \method achieves the high structural utility of the unpenalized baseline while simultaneously enforcing the sparsity necessary for communication efficiency.

\subsection{Structural Convergence: Federated vs. Local Learning}
\label{app:local_NOTEARS}
In Section~\ref{sec:expe}, we demonstrated that \method with participant-level personalization yields lower estimation errors (MSE) than local NOTEARS training. Figure~\ref{fig:local_vs_fed} provides further insights into this result as the structural level. 
\begin{figure*}[t]
    \centering  
\caption{Convergence of \method compared to local NOTEARS on a heterogeneous dataset ($P=20, n_p=50$).}
    \label{fig:local_vs_fed}
\includegraphics[width=0.8\textwidth]{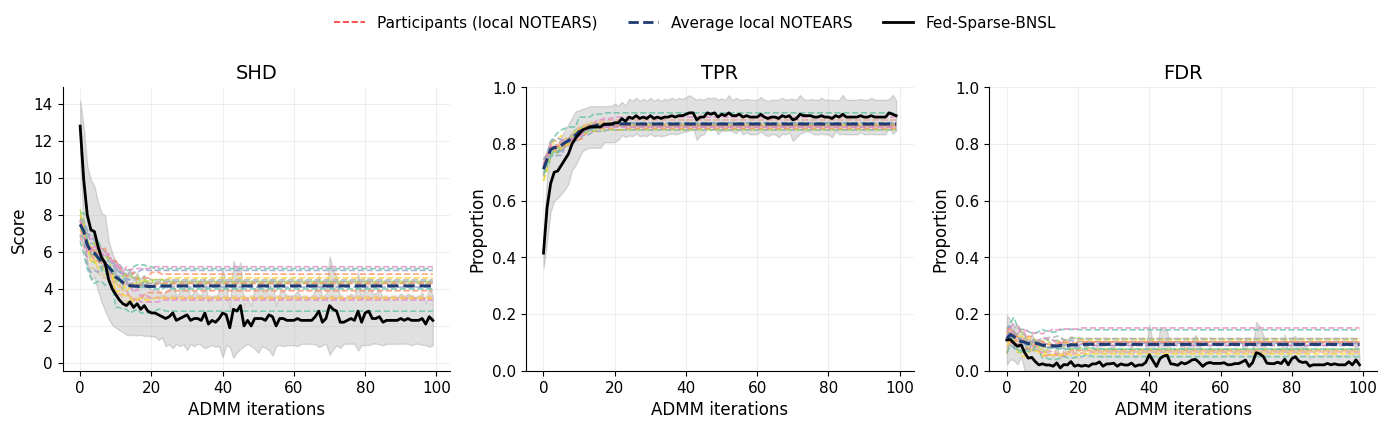} 
\end{figure*}
Although individual local NOTEARS runs converge to reasonably accurate DAGS, \method approach consistently attains lower SHD and FDR while maintaining comparable TPR. This indicates that aggregating information across participants improves structural recovery beyond what is achievable locally, even when local solutions are already moderately accurate.

\subsection{Privacy-Utility Trade-off for Smoothness Estimation}
\label{app:smoothness_expe}
We empirically evaluate the impact of allocating a fraction of the total privacy budget $\varepsilon$ to the private estimation of $M_{j}$. Let $\varepsilon_{\text{total}} = \varepsilon_M + \varepsilon$ be the total privacy budget, where $\varepsilon_M$ is the budget dedicated to estimating the smoothness constants, and $\varepsilon$ is the budget reserved for the main structure learning algorithm.
Figure~\ref{fig:privM} illustrates the end-to-end performance of \dpmethod with $\varepsilon_{\text{total}} = 5$, across different allocation ratios $\varepsilon_M / \varepsilon_{\text{total}} \in \{10\%, 20\%, 30\%, 40\%, 50\% \}$, compared to a baseline where the exact smoothness constants are assumed to be known (labeled as 'Baseline'). The results are reported over $10$ runs.
\begin{figure*}[t]
    \centering  
\includegraphics[width=1\textwidth]{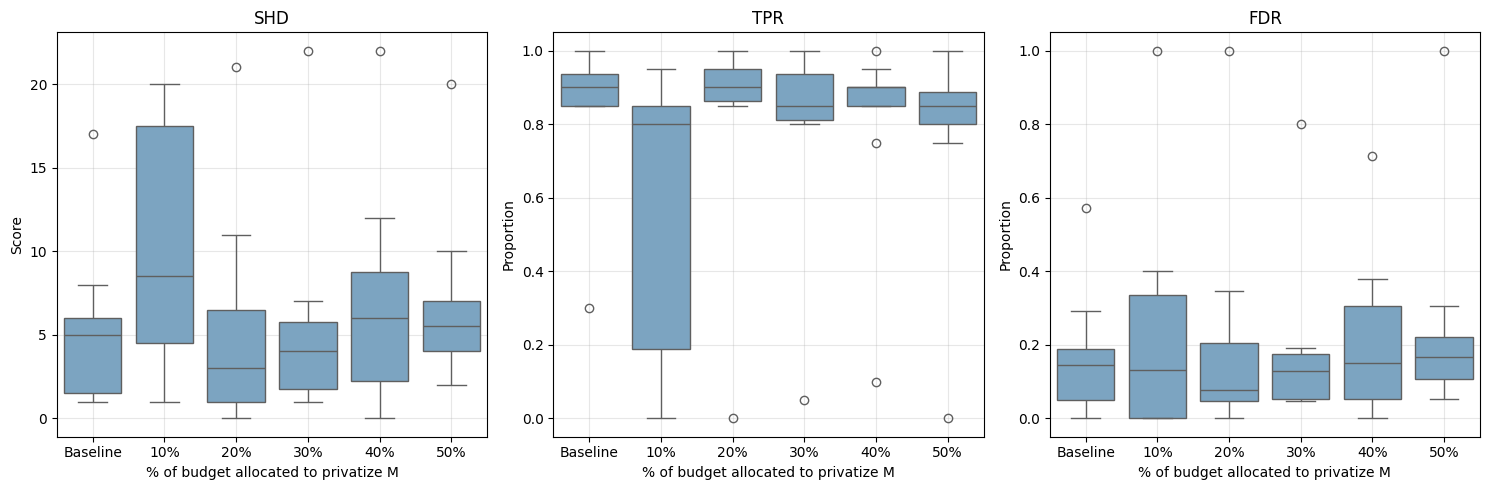} 
\caption{Utility (SHD, TPR, FDR) of \dpmethod with $\varepsilon_{\text{total}}=5$ for varying percentages of the total privacy budget allocated to privatize the smoothness constants $M_j$. 'Baseline' represents \dpmethod with $\varepsilon=5$ using exact, non-private smoothness constants.}
    \label{fig:privM}
\end{figure*}

The results reveal a clear privacy-utility trade-off governed by the budget split:
\begin{itemize}
    \item \textbf{Under-allocation ($10\%$)}: allocating too little budget to $\varepsilon_M$ results in highly noisy estimates of $M_j$. Since $M_j$ acts both as a scaling factor in the score function (Eq.~\ref{eq:score-non-private}) and as the inverse step size in the proximal gradient update (Eq.~\ref{eq:update}), the noise severely impacts the optimization process. This is clearly demonstrated by the large variance and severe degradation in all metrics. 
    \item \textbf{Over-allocation ($40\%, 50\%$)}: conversely, allocating too much budget to $\varepsilon_M$ yields more accurate smoothness constants but starves the main optimization routine $\varepsilon$. Consequently, the gradients and greedy coordinate selections become too noisy, leading to a degradation in structural accuracy.
    \item \textbf{Balanced allocation ($20\%, 30\%$)}: when allocating approximately $20\%$ to $30\%$ of the budget to $\varepsilon_M$, the algorithm achieves an optimal balance. The performance distributions in this regime are statistically comparable to \dpmethod with the non-private $M$ baseline.
\end{itemize}
\textbf{Conclusion.} This empirical evaluation confirms that the assumption of perfectly known smoothness constants can be safely relaxed in practice. By allocating a small portion of the privacy budget, participants can privately estimate their local constants without incurring any significant loss in the final utility.

\subsection{Empirical Comparison of DP-PGCD Variants}
\label{app:DPPGCD-expe}
\begin{figure}[!ht]
    \centering  
\includegraphics[width=1\textwidth]{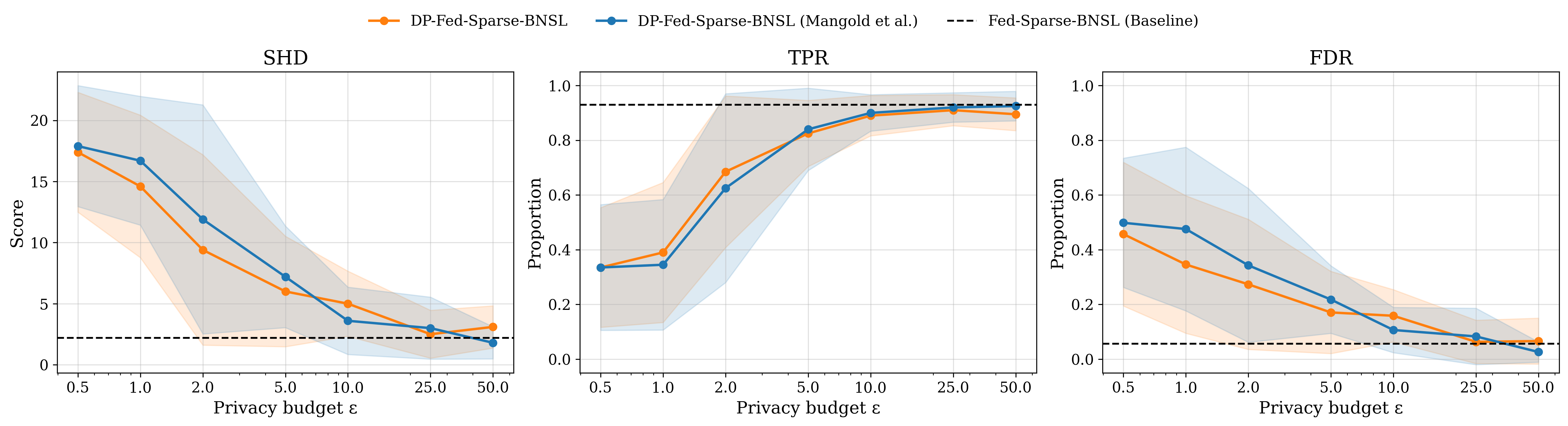} 
\caption{Empirical comparison of our \dpmethod vs. a variant using \citet{mangold}, and the non-private baseline.}
\label{fig:dppgcd}
\end{figure}

To complement the theoretical analysis showing that our improved DP-PGCD algorithm achieves a stronger privacy guarantee for the same amount of noise (see Appendix~\ref{app:DPPGCD-expe}, Figure~\ref{fig:dppgcd} shows empirical results comparing \dpmethod (using our DP-PGCD) with a variant of \dpmethod that uses the original DP-PGCD of \citet{mangold}, as well as with the non-private baseline. The comparison spans SHD, TPR, and FDR metrics as a function of the privacy budget $\varepsilon$. These results confirm that our method consistently achieves higher utility at the same privacy level.

\end{document}